\documentclass[lettersize,journal]{IEEEtran}
\usepackage{amsmath,amsfonts}
\usepackage{algorithmic}
\usepackage{algorithm}
\usepackage{array}
\usepackage{textcomp}
\usepackage{stfloats}
\usepackage{url}
\usepackage{verbatim}
\usepackage{graphicx}
\usepackage{cite}
\usepackage{multirow} 
\usepackage{booktabs}
\usepackage[normalem]{ulem} 
\usepackage[dvipsnames]{xcolor}
\definecolor{mygray}{gray}{0.6}
\usepackage{subfigure}
\usepackage{bbm} 

\usepackage{amssymb}
\usepackage{mathtools}
\usepackage{amsthm}

\usepackage[capitalize,noabbrev]{cleveref}

\theoremstyle{plain}
\newtheorem{theorem}{Theorem}[section]

\newtheorem{lemma}[theorem]{Lemma}

\theoremstyle{definition}

\newtheorem{assumption}[theorem]{Assumption}
\theoremstyle{remark}
\newtheorem{remark}[theorem]{Remark}

\begin{document}

\title{CCFC++: Enhancing Federated Clustering \\through Feature Decorrelation}

\author{Jie Yan,
        Jing Liu,
        Yi-Zi Ning and Zhong-Yuan Zhang\IEEEauthorrefmark{1}
\IEEEcompsocitemizethanks{\IEEEcompsocthanksitem *Corresponding author. E-mail addresses: zhyuanzh@gmail.com.
}}



\maketitle

\begin{abstract}
In federated clustering, multiple data-holding clients collaboratively group data without exchanging raw data. This field has seen notable advancements through its marriage with contrastive learning, exemplified by Cluster-Contrastive Federated Clustering (CCFC). However, CCFC suffers from heterogeneous data across clients, leading to poor and unrobust performance. Our study conducts both empirical and theoretical analyses to understand the impact of heterogeneous data on CCFC. Findings indicate that increased data heterogeneity exacerbates dimensional collapse in CCFC, evidenced by increased correlations across multiple dimensions of the learned representations. To address this, we introduce a decorrelation regularizer to CCFC. Benefiting from the regularizer, the improved method effectively mitigates the detrimental effects of data heterogeneity, and achieves superior performance, as evidenced by a marked increase in NMI scores, with the gain reaching as high as 0.32 in the most pronounced case.
\end{abstract}

\begin{IEEEkeywords}
Federated clustering, contrastive clustering, data heterogeneity, feature decorrelation.
\end{IEEEkeywords}

\section{Introduction}
\IEEEPARstart{F}{ederated} clustering (FC) extends traditional centralized clustering to federated scenarios, enabling multiple data-holding clients to collaboratively group data without sharing their raw data. It has gained relevance in applications such as client selection \cite{fu2023client} and personalization \cite{long2023multi, cho2023communication}. Naturally, adapting centralized clustering methodologies for federated scenarios has been a focus in this field.

In centralized clustering, significant progress has been largely attributed to the incorporation of representation learning techniques \cite{zhou2022comprehensive}. Parallel to this, FC has witnessed substantial advancements through its integration with representation learning, exemplified by Cluster-Contrastive Federated Clustering (CCFC)\cite{yan2024ccfc}. CCFC, which synergizes FC with contrastive learning \cite{oord2018representation, bachman2019learning}, has demonstrated marked improvements in clustering performance. However, this performance is adversely affected by data heterogeneity across clients, deteriorating with the increasing degree of data heterogeneity.

To comprehensively ascertain the impact of data heterogeneity on CCFC, we conducted both empirical and theoretical analyses comparing the representations learned under varying heterogeneity levels. These analyses consistently show that increased data heterogeneity exacerbates \textit{dimensional collapse} in CCFC, evidenced by heightened correlations across multiple dimensions of the learned representations (see \cref{case1}). Indeed, low inter-correlation across multiple dimensions of the learned representations is pivotal for the efficacy of various learning tasks, including clustering \cite{von2007tutorial, tao2021clustering}, self-supervised learning \cite{zbontar2021barlow, hua2021feature}, class incremental learning \cite{shi2022mimicking} and federated classification\cite{shi2023understanding}. Recognizing this, we propose a strategy to counter the adverse effects of data heterogeneity by addressing the dimensional collapse in the learned representations. We introduce a tailored decorrelation regularizer into the CCFC framework, resulting in an enhanced version named \textbf{CCFC++}.

\begin{figure}
\centering
\includegraphics[height = 6cm, width = 8cm]{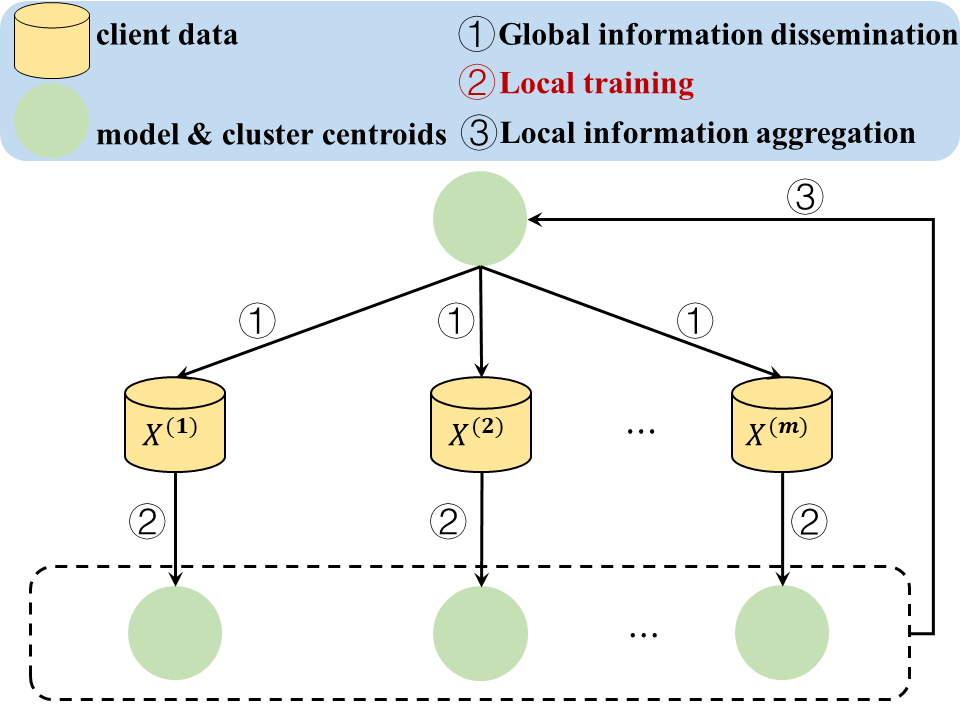}
\caption{\textbf{CCFC architecture.} In this work, we focus on the second step.}
\label{ccfc_arch}
\end{figure}

\begin{figure*}[!t]
\centering
\subfigure{
\includegraphics[height = 3.2cm, width = 3.1cm]{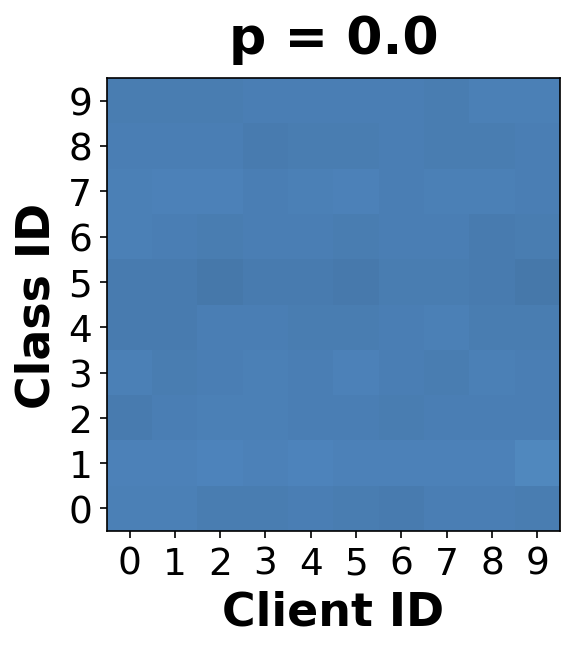}}
\quad
\subfigure{
\includegraphics[height = 3.2cm, width = 3.1cm]{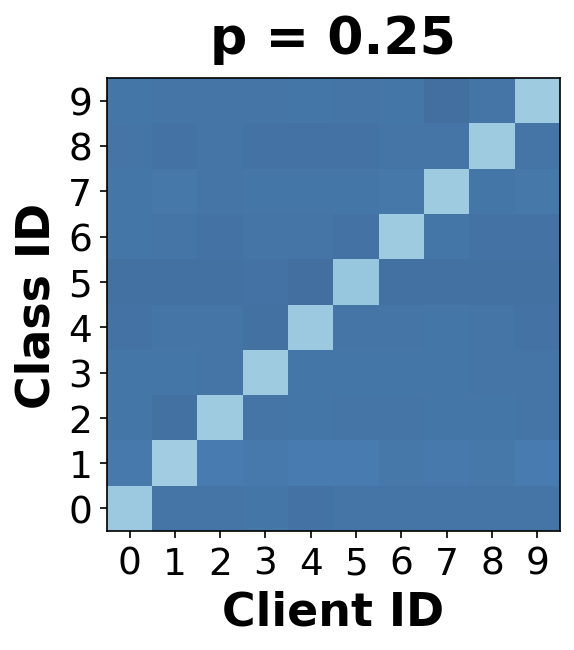}}
\,
\subfigure{
\includegraphics[height = 3.2cm, width = 3.1cm]{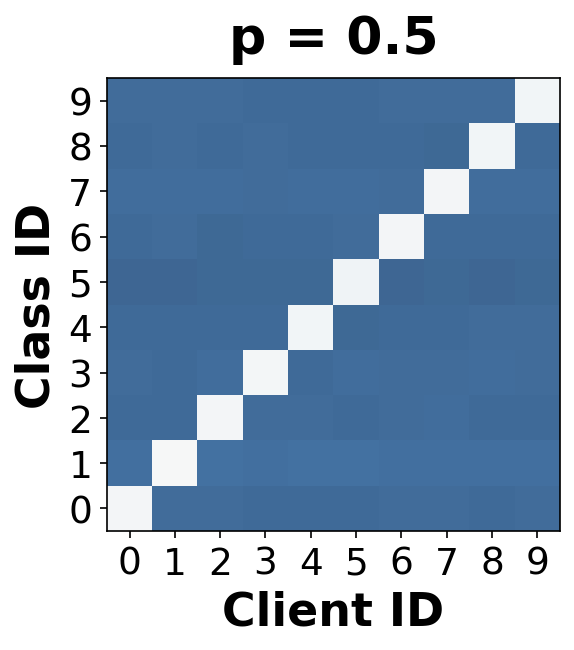}}
\,
\subfigure{
\includegraphics[height = 3.2cm, width = 3.1cm]{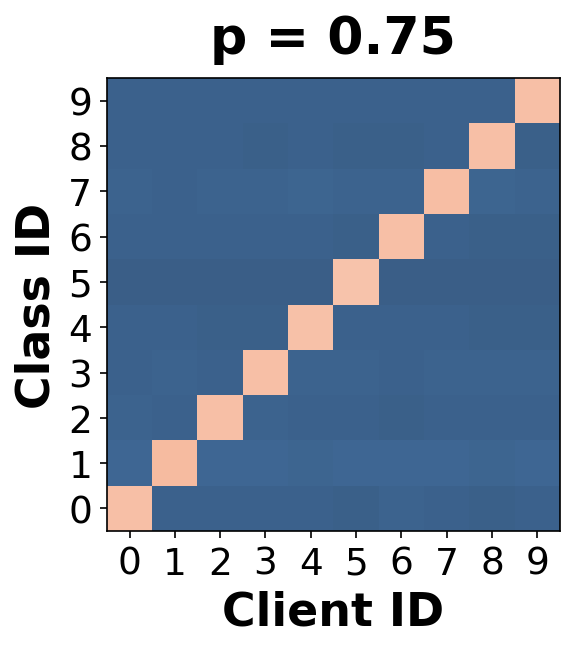}}
\,
\subfigure{
\includegraphics[height = 3.2cm, width = 3.4cm]{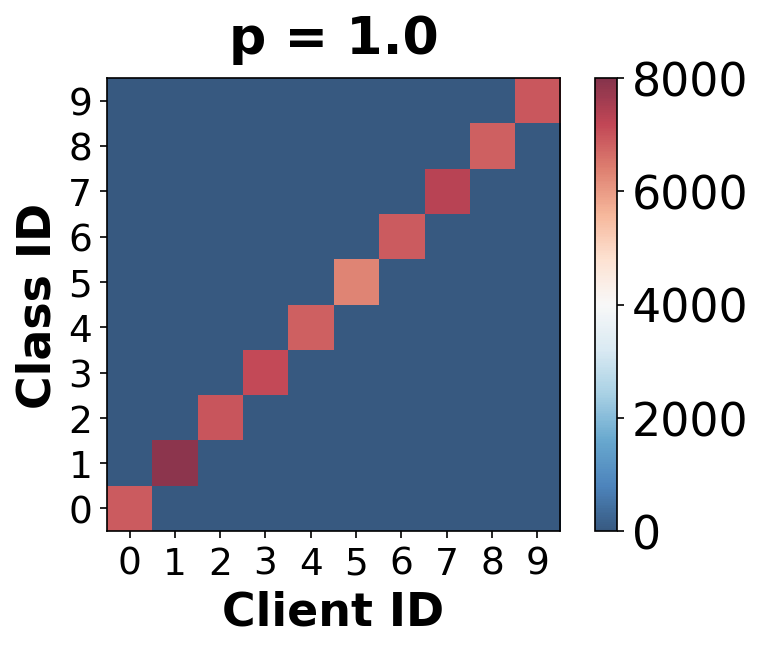}}

\subfigure{
\includegraphics[height = 3.2cm, width = 3.1cm]{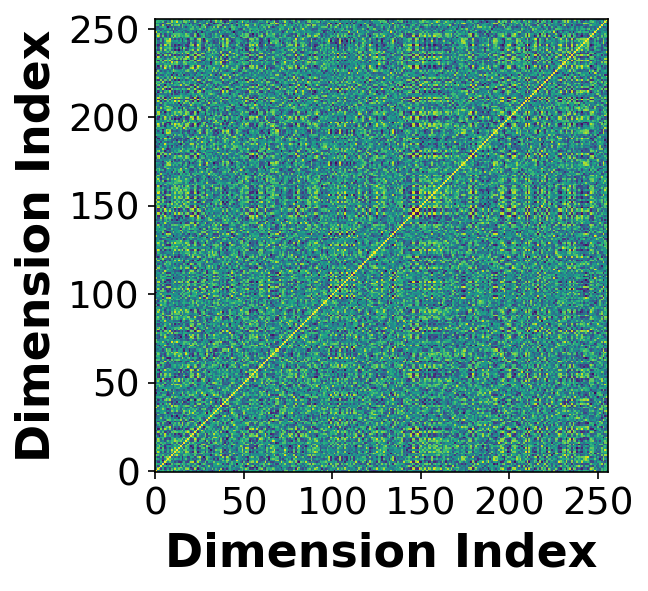}}
\,
\subfigure{
\includegraphics[height = 3.2cm, width = 3.1cm]{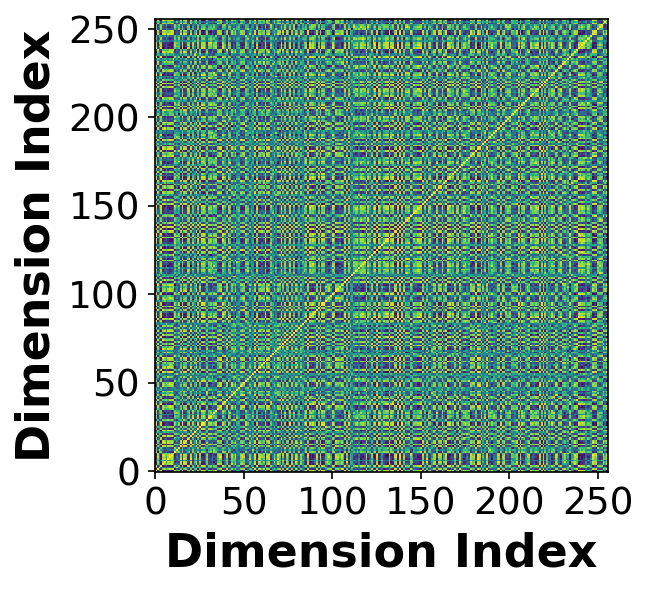}}
\,
\subfigure{
\includegraphics[height = 3.2cm, width = 3.1cm]{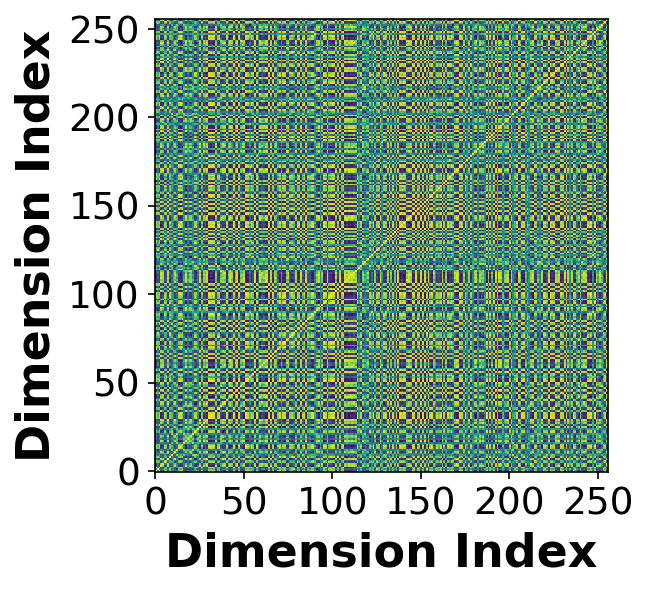}}
\,
\subfigure{
\includegraphics[height = 3.2cm, width = 3.1cm]{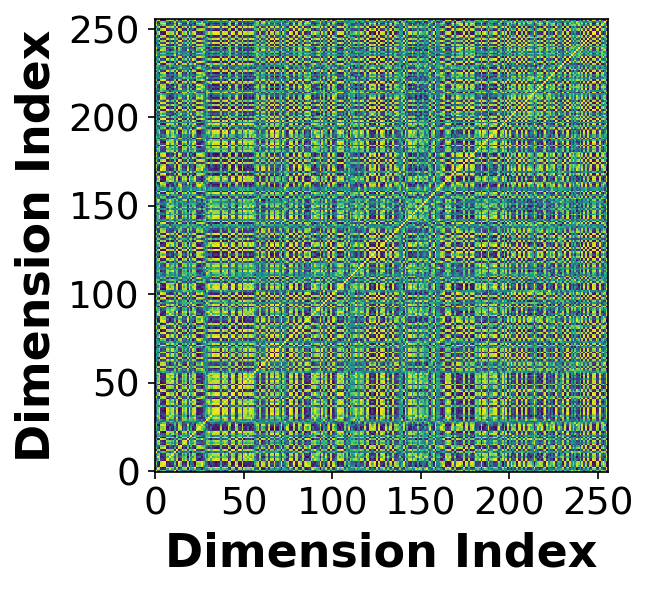}}
\,
\subfigure{
\includegraphics[height = 3.2cm, width = 3.4cm]{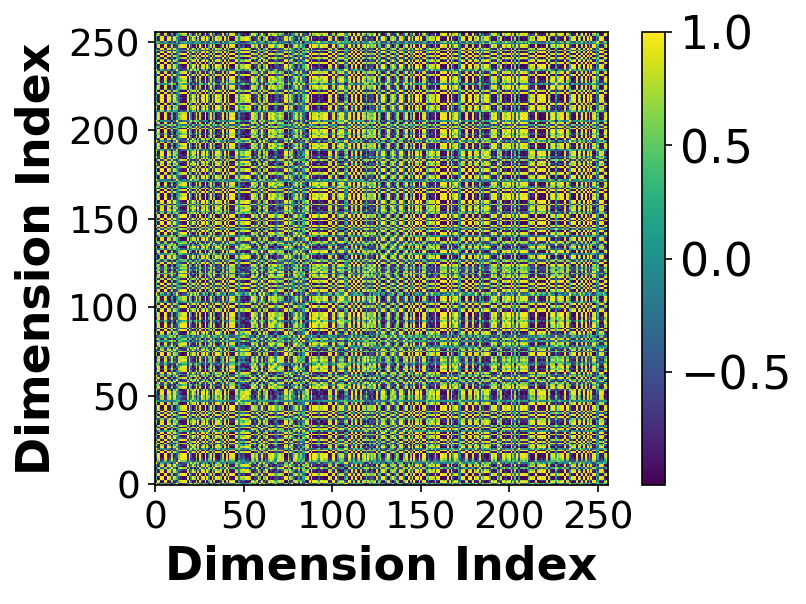}}

\subfigure{
\includegraphics[height = 3.2cm, width = 3.1cm]{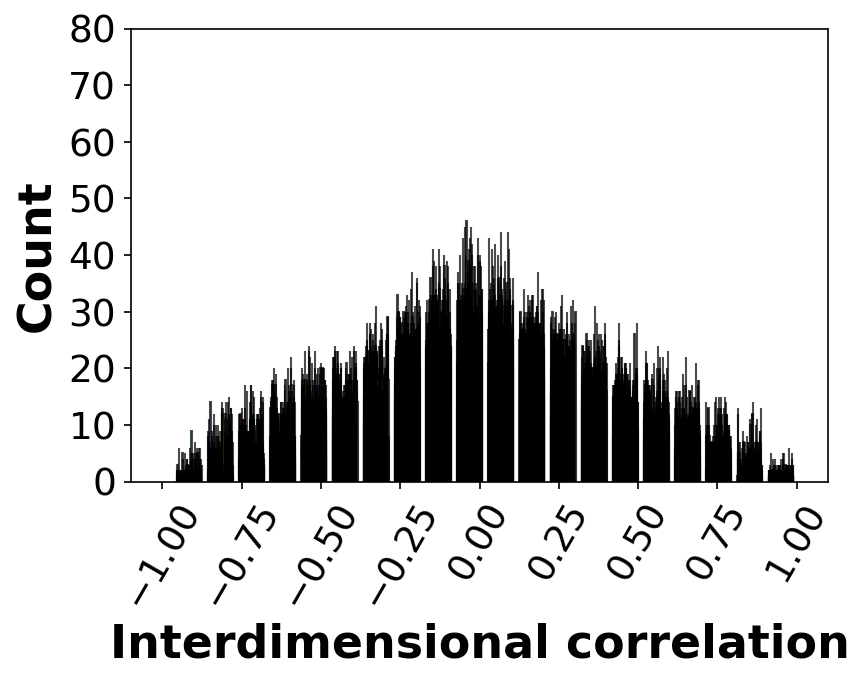}}
\,
\subfigure{
\includegraphics[height = 3.2cm, width = 3.1cm]{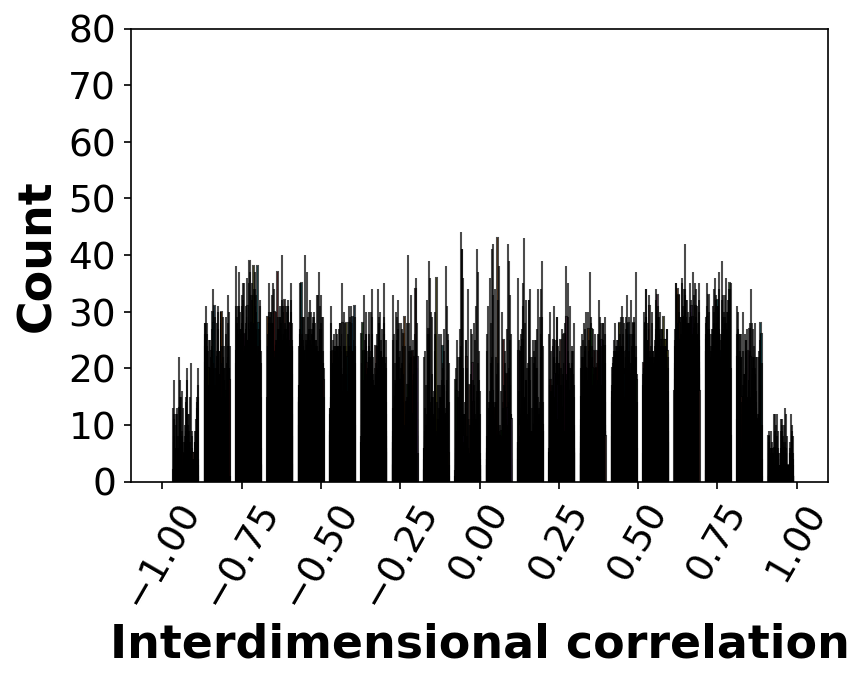}}
\,
\subfigure{
\includegraphics[height = 3.2cm, width = 3.1cm]{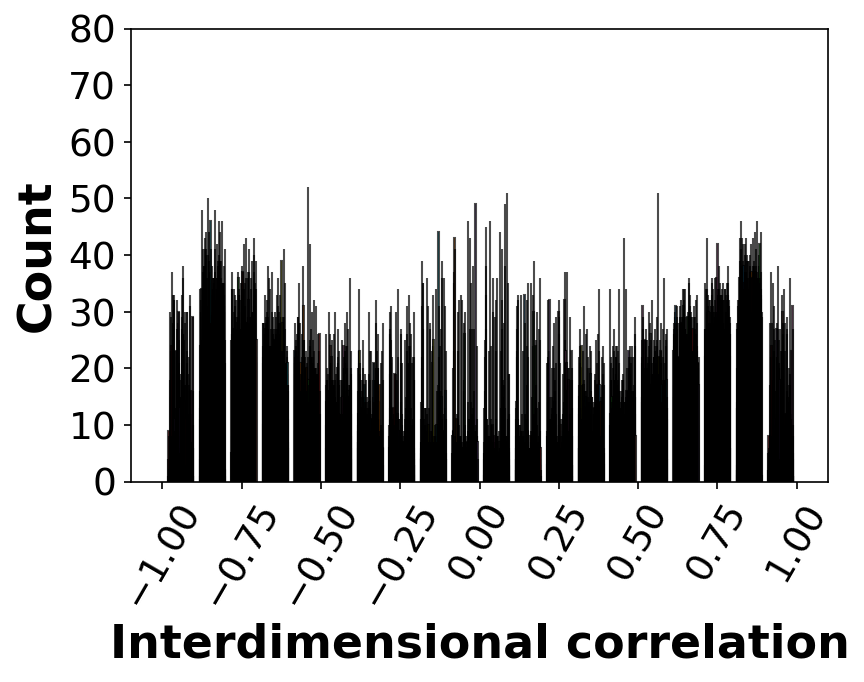}}
\,
\subfigure{
\includegraphics[height = 3.2cm, width = 3.1cm]{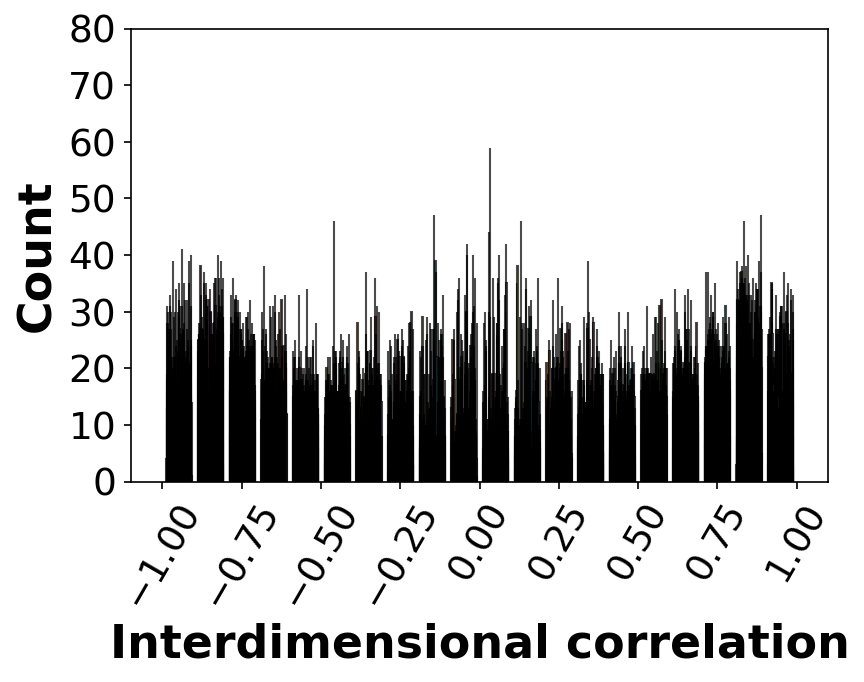}}
\,
\subfigure{
\includegraphics[height = 3.2cm, width = 3.1cm]{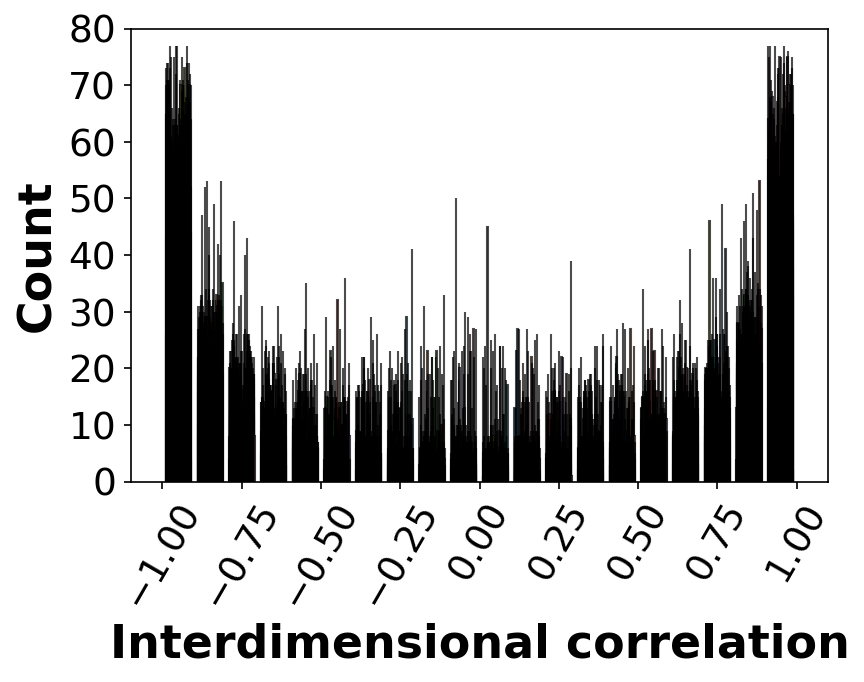}}

\caption{\textbf{The learned representations of CCFC under different simulated federated scenarios on MNIST (best viewed in color)}. \textbf{The first row} displays the local data distributions of each client under different levels of data heterogeneity. The color bar denotes the number of samples, and $p$ denotes the imbalance in classes across different clients. A larger $p$ implies stronger data heterogeneity. \textbf{The second row} shows the covariance matrices of the learned representations by CCFC in the corresponding federated scenarios. \textbf{The last row} showcases the distribution of interdimensional correlations in the corresponding covariance matrices.}
\label{case1}
\end{figure*}

Comprehensive experiments reveal that: 1) The decorrelation regularizer effectively mitigates the dimensional collapse, leading to more clustering-friendly representations. 2)CCFC++ significantly mitigates the negative impact of data heterogeneity, thereby achieving enhanced performance. In the most notable case, this modification led to an increase in the Normalized Mutual Information (NMI) \cite{strehl2002cluster} score by up to 0.32 and an improvement in the Kappa \cite{liu2019evaluation} score by as much as 0.27. 3) Beyond the data heterogeneity, this regularizer also demonstrates beneficial in handling systems heterogeneity \cite{li2020federated}. In summary, our contributions are threefold:
\begin{itemize}
\item
We provide both empirical and theoretical insights into how increased data heterogeneity exacerbates dimensional collapse in CCFC.

\item
Based on these insights, we enhance CCFC with a tailored decorrelation regularizer to address the challenges posed by data heterogeneity.

\item
We validate the effectiveness of the decorrelation regularizer and CCFC++ through extensive experiments.
\end{itemize}

\section{Background and related work}
Clustering, a cornerstone of unsupervised learning, has traditionally been studied within centralized scenarios where data aggregation on a central server is a foundational assumption \cite{ezugwu2022comprehensive, zhou2022comprehensive}. However, in many practical scenarios, data resides across multiple isolated clients, and privacy concerns often preclude the sharing or aggregation of this local data. Relying solely on local data for clustering tasks proves insufficient \cite{stallmann2022towards, yan2024ccfc}.

To address these, federated clustering (FC) has emerged, allowing multiple clients to collaboratively group data without exchanging raw data. As an extension of centralized clustering, FC inherently follows an exploratory trajectory, extending methodologies from centralized clustering to federated scenarios. Notable extensions include k-FED \cite{dennis2021heterogeneity} and SDA-FC-KM \cite{sdafc}, outgrowths of k-means \cite{lloyd1982least}; FFCM \cite{stallmann2022towards} and SDA-FC-FCM \cite{sdafc}, extensions of fuzzy c-means \cite{bezdek1984fcm}; and PPFC-GAN \cite{ppfcgan}, derived from DCN \cite{yang2017towards}. Although these extensions have advanced FC, a gap between FC and centralized clustering persists \cite{yan2024ccfc}.


A key driver of success in centralized clustering has been the integration of representation learning techniques \cite{zhou2022comprehensive}. Extending this methodology to FC offers a pathway to bridge this performance gap. In this vein, Cluster-Contrastive Federated Clustering (CCFC) \cite{yan2024ccfc} has been a pioneering approach, bridging FC with contrastive learning to achieve notable improvements in clustering performance. However, CCFC's effectiveness is compromised by data heterogeneity across clients, deteriorating with the increasing degree of data heterogeneity (\cref{NMI}).

In this work, building upon our exploration of how heterogeneous data affects CCFC, we ascertain that the adverse impact of data heterogeneity can be significantly mitigated through the incorporation of a decorrelation regularizer, diminishing the interdimensional correlation within the learned representations. We call this improved version of CCFC as CCFC++.

\section{CCFC++: unleashing CCFC's potential through feature decorrelation}

This section commences with a concise presentation of some preliminaries related to CCFC, followed by a comprehensive analysis of how heterogeneous data influences CCFC from both empirical and theoretical perspectives. Finally, we improve CCFC through the incorporation of a tailored regularizer.

\subsection{CCFC}
\subsubsection{Overview of the CCFC Architecture}
Given a real-world dataset $X$ distributed among $m$ clients, i.e., $X=\bigcup_{l=1}^{m} X^{(l)}$. Our goal is to divide the samples into $k$ clusters, with high intra-cluster similarity and low inter-cluster similarity, without sharing the raw data.

As shown in \cref{ccfc_arch}, throughout the entire training process of CCFC, the only shared information between clients and the server are models and cluster centroids, thereby safeguarding data privacy. The training process in CCFC involves three primary steps per communication round:
\begin{itemize}
\item
1) Global information dissemination. Each client downloads the global information from the server and updates their local models accordingly.

\item
2) Local training. Each client first assign their local data to the nearest global cluster centroid, followed by local model training using these labeled data. Then, each client also computes $k$ local cluster centroids using k-means (KM) \cite{lloyd1982least} to the learned representations, capturing local semantic information.

\item
3) Local information aggregation. Each client uploads their local information to the server, where the local models are aggregated into a new global model through weighted averaging, and the local cluster centroids are aggregated into $k$ new global cluster centroids using KM, for the next communication round.
\end{itemize}
When we complete the scheduled communication rounds, the final clustering result can be obtained by assigning data to the closest global cluster centroid. In this work, we will focus on the local training step, since the representation learning process mainly occurs in this step.


\subsubsection{The local training step}
In CCFC, the model for sharing and training is a tailored cluster-contrastive model, which aims to learn cluster-invariant representations, meaning that samples within the same cluster should have similar representations. The cluster-contrastive model $w$ comprises a encoder $f$ and an MLP predictor $h$. The encoder $f$ comprises a backbone (e.g., ResNet-18 \cite{he2016deep}) and an MLP projector \cite{chen2020simple}.

Given the downloaded global model $w^{(g)} = (f^{(g)},\, h^{(g)})$, $k$ global cluster centroids $\{\eta^{(g,\, c)}\}_{c = 1}^{k}$, and the updated local model $w^{(l)} = (f^{(l)},\, h^{(l)})$. The local data of each client $l$ can be labeled with the index of the nearest global centroid:
\begin{equation}
\underset{c=\{1,\, \cdots,\, k\}}{\arg \min }\left\|f^{(g)}(x) - \eta^{(g,\, c)}\right\|_{2},
\label{label_assign}
\end{equation}
where $\|\cdot\|_2$ is $\ell_2$-norm.
Then, each client $l$ trains $w^{(l)}$ with their local data $X^{(l)}$ and the labeled results. The loss function is defined as:
\begin{align}
\ell = &\sum_{c = 1}^{k}\frac{1}{kn_c^2}\sum_{i = 1}^{n_c}\sum_{j = 1}^{n_c} D(p_{i}^{(c)},\, stopgrad(z_{j}^{(c)}))  \nonumber\\
&+\sum_{c = 1}^{k} \frac{\lambda}{kn_c} \sum_{i = 1}^{n_c}D(p_{i}^{(c)},\, stopgrad(p_{i}^{(g,\, c)})),
\label{my_sim}
\end{align}
where $D(\cdot,\, \cdot)$ is the negative cosine similarity function, the stop-gradient operation ($stopgrad(\cdot)$) is an critical component to avoid model collapse \cite{yan2024ccfc}, $n_c$ is the number of samples in the $c$-th cluster, $p_{i}^{(c)}  = h^{(l)}(f^{(l)}(x_{i}^{(c)}))$ and $p_{i}^{(g,\, c)} = h^{(g)}(f^{(g)}(x_{i}^{(c)}))$ are the predictions of sample $x_{i}^{(c)}$ for the latent representations of samples $\{x_{j}^{(c)}\}_{j = 1}^{n_c}$ within the same cluster, $z_{j}^{(c)} = f^{(l)}(x_{j}^{(c)})$ is the latent representation of sample $x_{j}^{(c)}$, and $\lambda$ is the tradeoff hyperparameter. By minimizing \cref{my_sim}, the first item will encourage the local model to learn cluster-invariant representations for samples within the same cluster, and the second one will encourage the local model not to deviate too far from the global model.

Despite CCFC showcases substantial enhancement in clustering performance, the clustering performance suffers from the data heterogeneity problem, deteriorating with the increasing degree of data heterogeneity. To comprehensively ascertain how heterogeneous data affects CCFC, we will empirically and theoretically compare the representations learned by CCFC under different levels of data heterogeneity.


\subsection{Empirical observations on the global model}
\label{cg}
We first empirically demonstrate the dimensional collapse problem on the global model. Specifically, we first partition the samples of MNIST into $k^\star$ subsets corresponding to different clients, where $k^\star$ is the number of true clusters ($k^\star = 10$ for MNIST).  Then, following \cite{chung2022federated, yan2024ccfc}, we simulate federated scenarios with varying class imbalances across clients, controlled by a hyperparameter $p$. For the $l$-th client with $s$ images, $p\cdot s$ images are sampled from the $l$-th cluster, while the remaining $(1 - p) \cdot s$ images are drawn randomly from the entire data. The hyperparameter $p$ varies from 0 (data is randomly distributed among $m$ clients) to 1 (each client forms a cluster). We let $p \in \{0,\, 0.25,\, 0.5,\, 0.75,\, 1\}$.

For each federated scenario, we first train a CCFC, then calculate the covariance matrix of the representations learned by the global model, and visualize this matrix and the distribution of elements within it. As shown in \cref{case1}, under different levels of data heterogeneity, the learned representations \textit{all} face the problem of \textit{dimensional collapse},  i.e. multiple dimensions of the learned representations exhibit correlations. And this problem exacerbates with the increase in data heterogeneity.

To provide a more comprehensive characterization of the distribution of the learned representations, we also perform singular value decomposition (SVD) on these covariance matrices, and visualize the singular values. \cref{case2} further corroborates the observations in \cref{case1}.

\begin{figure}[!t]
\centering
\includegraphics[height = 6cm, width = 7cm]{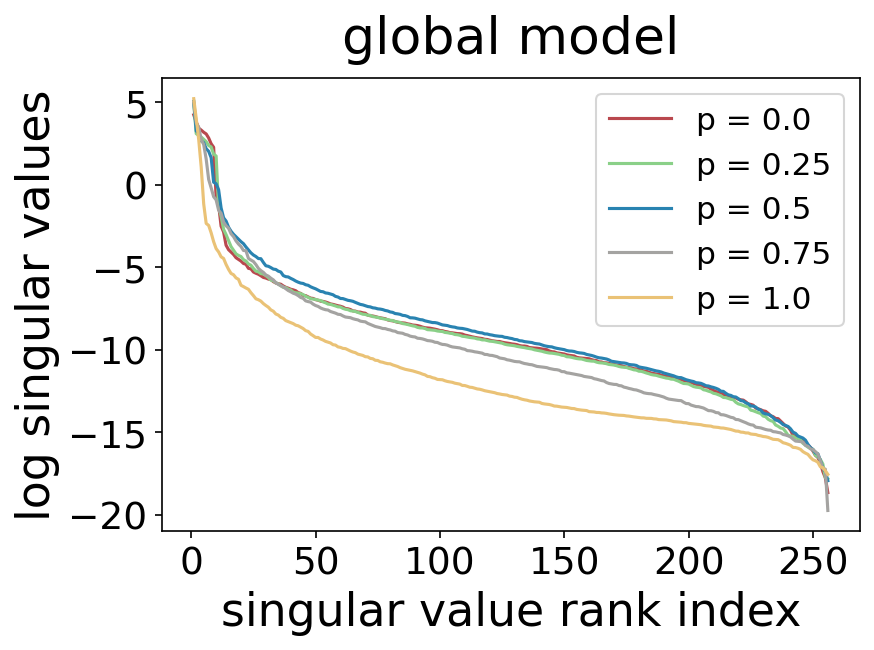}

\caption{\textbf{Dimensional collapse on the global model}. There are a considerable number of singular values collapsing to zero for all scenarios, implying collapsed dimensions. And this problem exacerbates with the increase in data heterogeneity.}
\label{case2}
\end{figure}

\subsection{Empirical observations on local models}
Since the global model is derived by aggregating local models, we conjecture that the dimensional collapse observed on the global model is attributable to the analogous problem on local models.

To substantiate this, we also empirically demonstrate the dimensional collapse problem on local models. Specifically, we use the same settings of federated scenarios and model training as in \cref{cg}. For each federated scenario, we first calculate the covariance matrix of the representations learned by a randomly selected local model. Then, we perform SVD on these covariance matrices, and visualize the singular values. As shown in \cref{case3}, we also observe similar collapse problem on the local model, which corroborates our conjecture.

\subsection{Theoretical analysis for dimensional collapse}
Based on our empirical observations, we have corroborated that the problem of dimensional collapse on the global model stems from local models. Hence, in this section, we focus on analyzing the dimensional collapse on local models, and theoretically elucidate how increased data heterogeneity causes the model weights to evolve into low-rank, leading to more severe dimensional collapse.

\subsubsection{Setups and notations}
For the generality and simplicity, our subsequent analyses will delve into the local training of an arbitrary client, disregarding the client ID. Given $k$ clusters obtained by labeling the local data with the index corresponding to the nearest global cluster centroid, each one comprises $n_c$ $(c \in [k] = \{1,\, 2,\, \cdots,\, k\})$ samples. We denote the collection of samples in cluster $c$ as $X^{(c)} = [x_{1}^{(c)},\, x_{2}^{(c)},\, \cdots,\, x_{n_c}^{(c)}] \in \mathbb{R}^{d \times n_c}$, the $o$-th feature of $x_{i}^{(c)}$ as $x_{oi}^{(c)} \in \mathbb{R}$, the collection of the $o$-th features in cluster $c$ as $x_{o\cdot}^{(c)} \in \mathbb{R}^{1 \times n_c}$ (i.e. the $o$-th row vector of $X^{(c)}$), where $d$ is the dimension of the sample, $o \in [d], i \in [n_c]$.

To analyze the dynamics of neural networks, a commonly used framework is gradient flow dynamics \cite{arora2018optimization, arora2019implicit, jing2021understanding, shi2023understanding}. Following these works, we also assume the cluster-contrastive model is a multi-layer linear neural network. The model comprises a linear neural network with $L_1 + L_2$ layers ($L_1 \geq 1$ and $L_2 \geq 1$), wherein the first $L_1$ layers correspond to the encoder, and the last $L_2$ layers correspond to the predictor.

At the optimization time step $t$, we denote the weight matrix of the $i$-th layer as $W_i(t)$, where $i \in [L_1 + L_2]$. Then, the weight matrices of the encoder and the predictor are respectively denoted as:
\begin{align}
\Pi(t) &= W_{L_1}(t)W_{L_1-1}(t)\cdots W_{1}(t) \in \mathbb{R}^{d' \times d},
\label{pi1}\\
\Phi(t)&= W_{L_2}(t)W_{L_2-1}(t)\cdots W_{L_1 + 1}(t) \in \mathbb{R}^{d' \times d'},
\label{pi2}
\end{align}
where $d'$ is the dimension of both the latent representation and the prediction.
We denote the local representations of $X^{(c)}$ as $Z^{(c)}(t) = \Pi(t) X^{(c)} = [z_{1}^{(c)}(t),\, z_{2}^{(c)}(t),\, \cdots,\, z_{n_c}^{(c)}(t)] \in \mathbb{R}^{d' \times n_c}$, the corresponding predictions as $P^{(c)}(t) = \Phi(t) Z^{(c)}(t) = \Phi(t)\Pi(t) X^{(c)} = [p_{1}^{(c)}(t),\, p_{2}^{(c)}(t),\, \cdots,\, p_{n_c}^{(c)}(t)] \in \mathbb{R}^{d' \times n_c}$. Similarly, the global predictions of $X^{(c)}$ can be denoted as:  $P^{(g,\, c)} = \Phi^{(g)} Z^{(g,\, c)} = \Phi^{(g)}\Pi^{(g)} X^{(c)} = [p_{1}^{(g,\, c)},\, p_{2}^{(g,\, c)},\, \cdots,\, p_{n_c}^{(g,\, c)}] \in \mathbb{R}^{d' \times n_c}$, where $\Phi^{(g)}$ and $\Pi^{(g)}$, the global weight matrices received from the server, remain fixed during the local training process. We summarize these notations in \cref{notations} of the \cref{prove}.

The gradient descent dynamics of  $\Pi$ and $\Phi$ are respectively denoted as:
\begin{align}
\dot{\Pi}(t) &= -\frac{\partial \ell(\Pi(t),\, \Phi(t))}{\partial \Pi},\\
\dot{\Phi}(t) &= -\frac{\partial \ell(\Pi(t),\, \Phi(t))}{\partial \Phi},
\end{align}
where $\ell$ is the loss function defined in \cref{my_sim}.


\begin{figure}[!t]
\centering
\includegraphics[height = 6cm, width = 7cm]{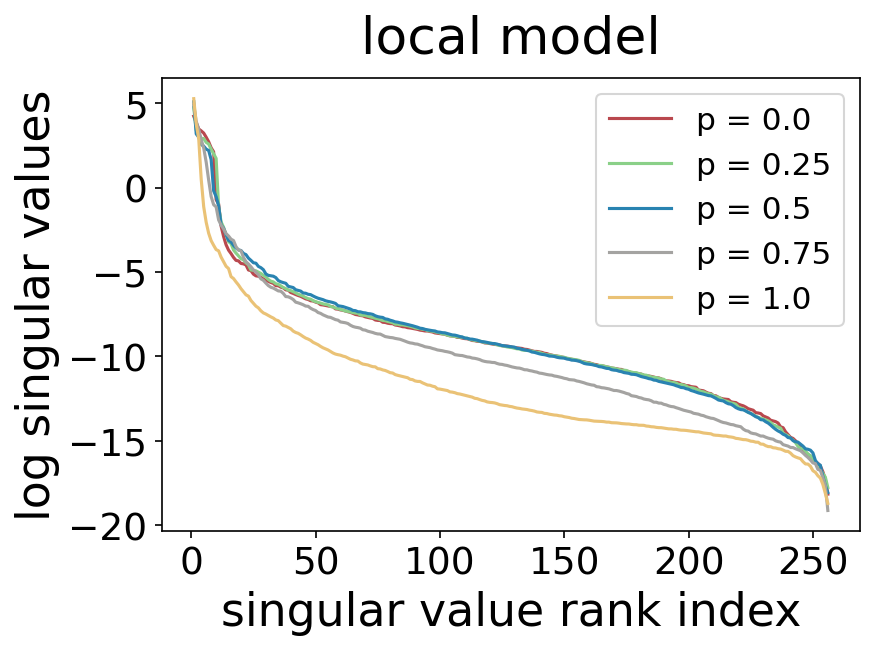}

\caption{\textbf{Dimensional collapse on the local model}. There are a considerable number of singular values collapsing to zero for all scenarios, implying collapsed dimensions. And this problem exacerbates with the increase in data heterogeneity.}
\label{case3}
\end{figure}

\subsubsection{Analysis on gradient flow dynamics}
Since our goal is to analyze the learned representations and it is directly produced by the encoder, we focus on the dynamic evolution of the weight matrix $\Pi(t)$. Specifically, we derive the dynamic evolution of the singular values of $\Pi(t)$, as shown in the theorem below.

\begin{assumption}
Assuming at the optimization time step $0$, $W_{i}(0)(W_{i}(0))^{\top}=(W_{i + 1}(0))^{\top}W_{i + 1}(0)$ holds for any layer $i \in [L_1 - 1]$.
\label{ass1}
\end{assumption}

\begin{assumption}
Assuming $|(u_{\tau}^{\Pi}(t))^\top v^{\Phi}_{{\tau}'}(t)|=\mathbbm{1}\{{\tau} = {\tau}'\}$ holds for any optimization time step, where $u_{\tau}^{\Pi}(t)$ is the $\tau$-th left singular vector of $\Pi(t)$ and $v_{{\tau}'}^{\Phi}(t)$ is the $\tau'$-th right singular vector of $\Phi(t)$.
\label{ass2}
\end{assumption}
\begin{remark}
\cref{ass1} can be achieved through appropriate weight initialization methods. And \cref{ass2} can be achieved by the gradient descent optimization under some assumptions \cite{ji2018gradient}.
\end{remark}

\begin{theorem}
Under assumptions \cref{ass1} and \cref{ass2}, the gradient descent dynamics of the $\tau$-th largest singular value $\sigma_{\tau}^{\Pi}(t)$ of $\Pi(t)$ can be expressed as:
\begin{align}
\dot{\sigma}_{\tau}^{\Pi}(t)  = &L_1(\sigma_\tau^{\Pi}(t))^{2-\frac{2}{L_1}}
\sqrt{(\sigma_{\tau}^{\Pi}(t))^{\frac2{L_1}} + C} \nonumber\\
& \times
(u_{\tau}^{\Phi}(t))^\top
\bar{Q}(t)
v_\tau^{\Pi}(t),
\label{eq_theory}
\end{align}
where $u_{\tau}^{\Phi}(t)$ is the $\tau$-th left singular vector of $\Phi(t)$, $v_{{\tau}}^{\Pi}(t)$ is the $\tau$-th right singular vector of $\Pi(t)$, $C$ is a constant,
\begin{equation}
\bar{Q}(t) = \frac1k\sum_{c = 1}^{k}Q^{(c)}(t) (X^{(c)})^\top,
\label{Q}
\end{equation}
the element $q_{ri}^{(c)}(t) \in \mathbb{R}$ located in the $r$-th $(r \in [d'])$ row and $i$-th $(i \in [n_c])$ column of $Q^{(c)}(t) \in \mathbb{R}^{d' \times n_c}$ is denoted as:
$q_{ri}^{(c)}(t) = \left(\frac{1}{n_c^2}\sum_{j = 1}^{n_c} \hat{z}_{rj}^{(c)}(t) + \frac{\lambda}{n_c} \hat{p}_{ri}^{(g,\, c)}  \right)
\cdot \frac{1 - (\hat{p}_{ri}^{(c)}(t))^2}{\|p_{i}^{(c)}(t)\|_2}$, $\hat{p}_{ri}^{(c)}(t) = \frac{p_{ri}^{(c)}(t)}{\|p_{i}^{(c)}(t)\|_2} \in \mathbb{R}$,
$\hat{z}_{rj}^{(c)}(t) = \frac{z_{rj}^{(c)}(t)}{\|z_{j}^{(c)}(t)\|_2} \in \mathbb{R}$, and $\hat{p}_{ri}^{(g,\, c)} = \frac{p_{ri}^{(g,\, c)}}{\|p_{i}^{(g,\, c)}\|_2} \in \mathbb{R}$.
\label{theory}
\end{theorem}
The detailed proof process is provided in \cref{prove}.

%
%

\begin{figure*}[!t]
\centering
\subfigure{
\includegraphics[height = 4cm, width = 4cm]{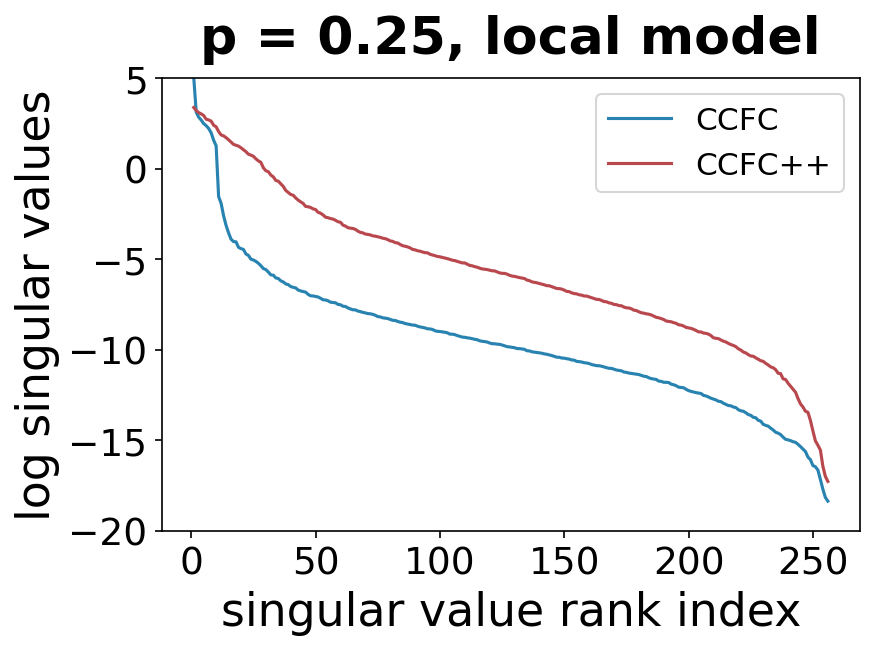}}
\,
\subfigure{
\includegraphics[height = 4cm, width = 4cm]{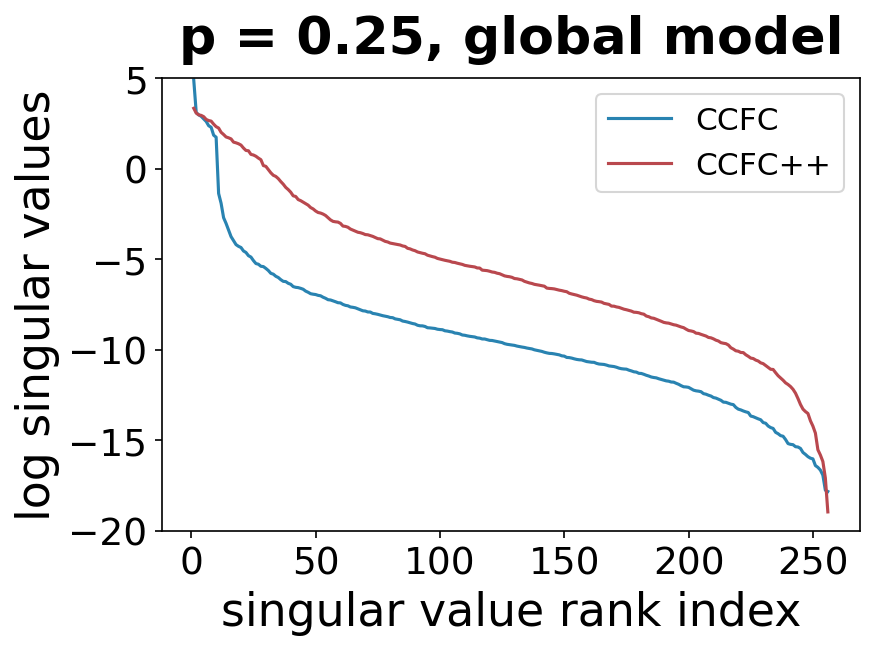}}
\,
\subfigure{
\includegraphics[height = 4cm, width = 4cm]{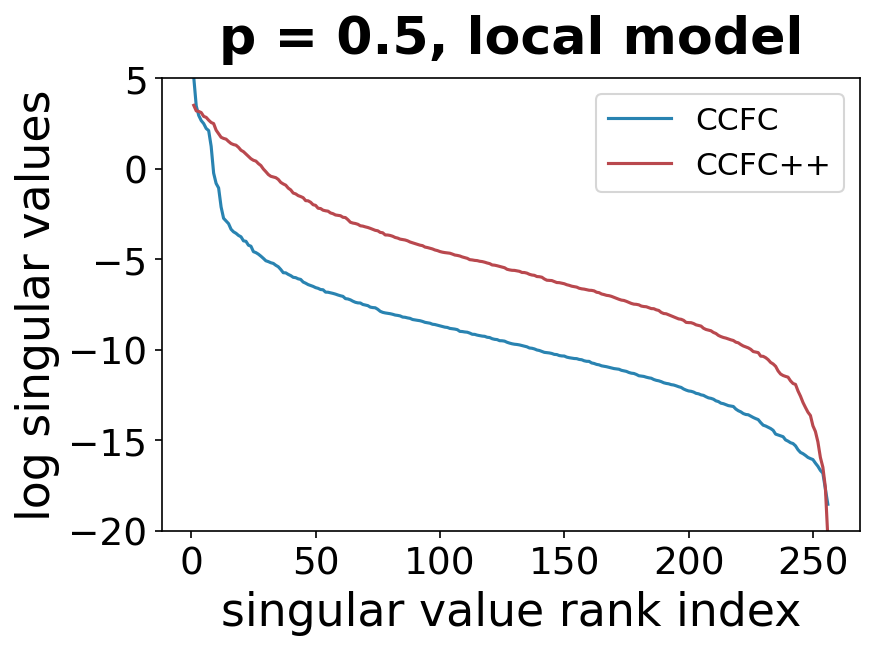}}
\,
\subfigure{
\includegraphics[height = 4cm, width = 4cm]{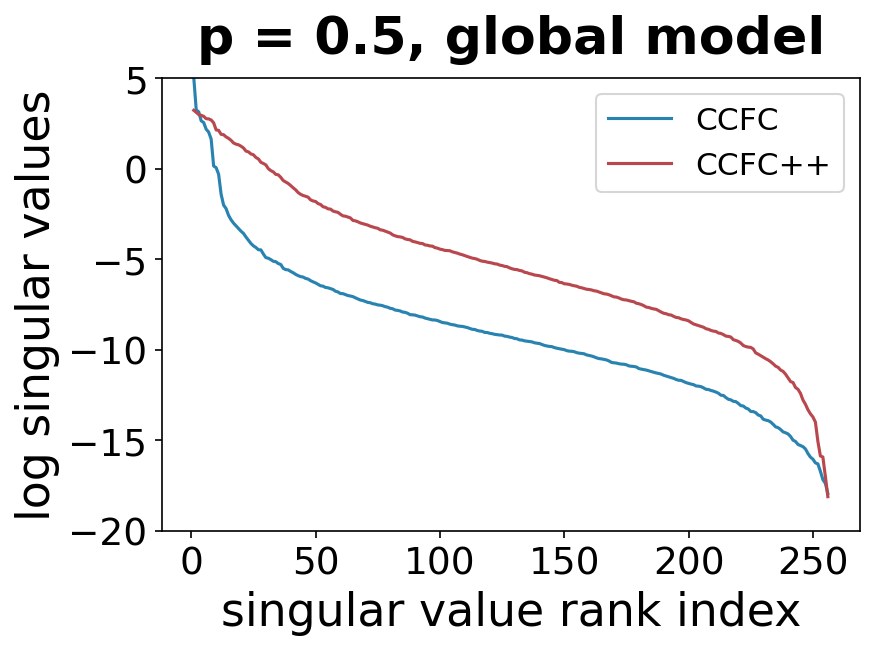}}

\subfigure{
\includegraphics[height = 4cm, width = 4.1cm]{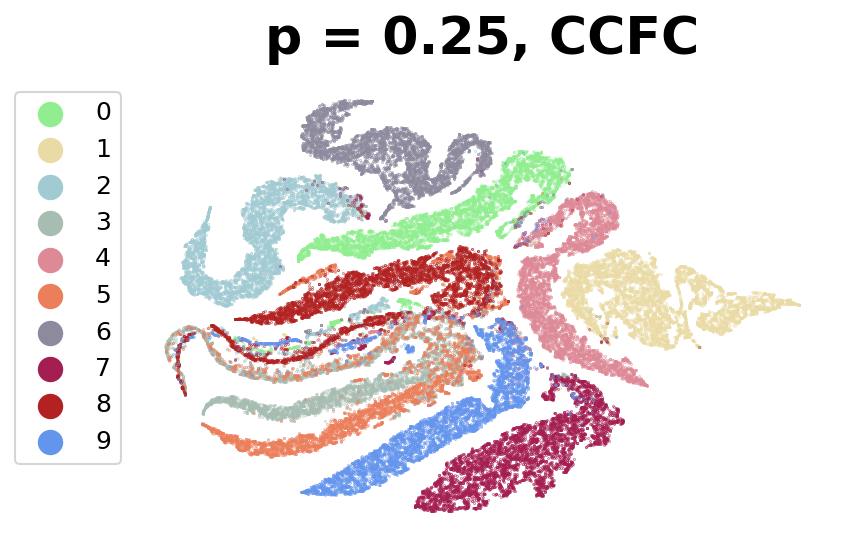}}
\,
\subfigure{
\includegraphics[height = 4cm, width = 4cm]{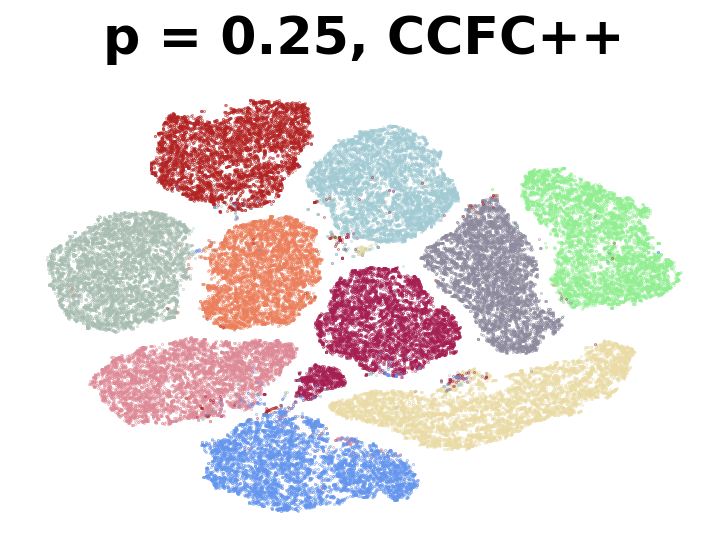}}
\,
\subfigure{
\includegraphics[height = 4cm, width = 4cm]{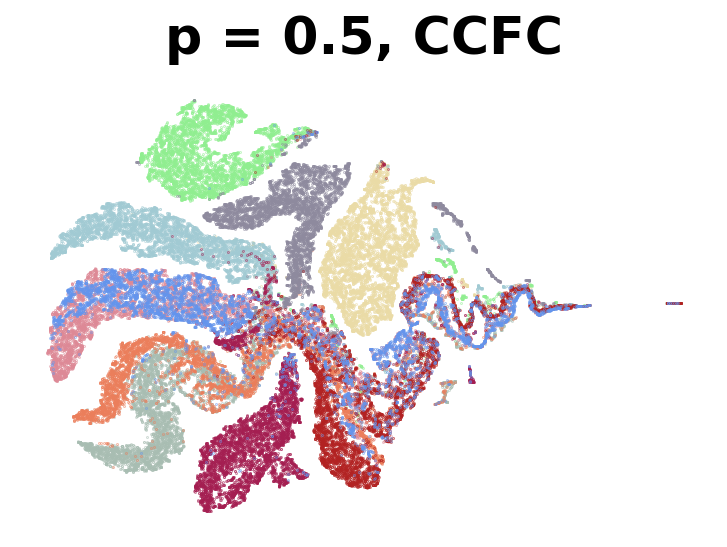}}
\,
\subfigure{
\includegraphics[height = 4cm, width = 4cm]{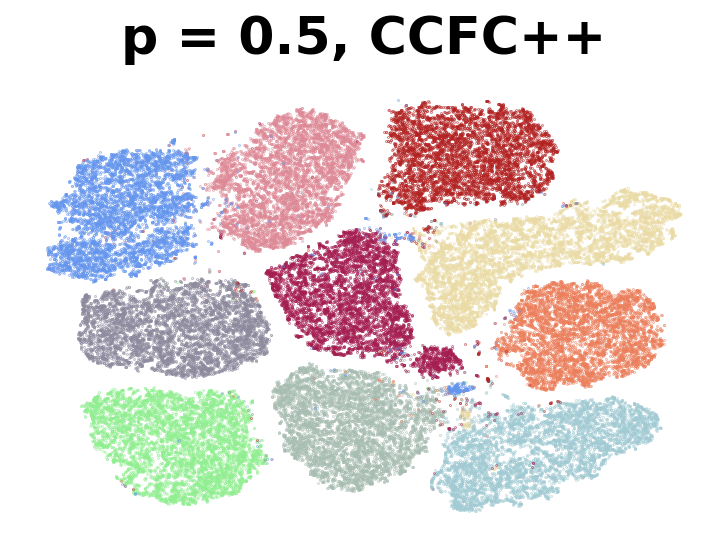}}

\caption{\textbf{The efficacy of the decorrelation regularizer under different simulated federated scenarios on MNIST}.
\textbf{The first row} plots the singular values of the covariance matrix of the learned representations. \textbf{The second row} showcases the learned representations of the global model of CCFC and CCFC++. Each color corresponds to a true cluster. A larger $p$ implies stronger data heterogeneity.} 
\label{case4}
\end{figure*}

Drawing upon \cref{theory}, we now explain why stronger data heterogeneity leads to $\Pi(t)$ being of lower rank. Note that increased data heterogeneity implies a more pronounced imbalance among the true clusters within the client  (recall \cref{case1}), which concomitantly leads to the predicted clusters exhibiting more marked similarities, both intra-cluster and inter-cluster, and a lower-rank $\bar{Q}(t)$ (defined in \cref{Q}). Furthermore, due to the orthogonality of $u_{\tau}^{\Phi}(t)$ and $v_\tau^{\Pi}(t)$ across different $\tau$'s, the term $(u_{\tau}^{\Phi}(t))^\top
\bar{Q}(t)
v_\tau^{\Pi}(t)$ in \cref{eq_theory} tends to be insignificant (small in magnitude) for more values of $\tau$. Then, $\dot{\sigma}_{\tau}^{\Pi}(t)$, the evolving rate of  ${\sigma}_{\tau}^{\Pi}(t)$  will be small for most of the $\tau$’s. As a result, only a few singular values of ${\Pi}(t)$ will exhibit a marked increase after training, resulting in a low-rank ${\Pi}(t)$.

Moreover, the covariance matrix of the representations can be rewritten as:
\begin{align}
\Sigma(t) &= \frac1{kn_c}\sum_{c=1}^k\sum_{i=1}^{n_c}(z_i^{(c)}(t)-\bar{z}(t))
(z_i^{(c)}(t)-\bar{z}(t))^\top \nonumber\\
&= \Pi[\frac1{kn_c}\sum_{c=1}^k\sum_{i=1}^{n_c}(x_i^{(c)}(t)-\bar{x}(t))
(x_i^{(c)}(t)-\bar{x}(t))^\top]\Pi^\top
\end{align}
where $\bar{z}(t) = \frac1{kn_c}\sum_{c=1}^k\sum_{i=1}^{n_c}z_i^{(c)}(t)$, and $\bar{x}(t) = \frac1{kn_c}\sum_{c=1}^k\sum_{i=1}^{n_c}x_i^{(c)}(t)$. Obviously, a low-rank ${\Pi}(t)$ can lead to a low-rank $\Sigma(t)$, which means the dimensional collapse for the representations.

\subsection{CCFC++}
Indeed, low inter-correlation across multiple dimensions of the learned representations is crucial for many learning tasks to attain superior performance, such as clustering \cite{von2007tutorial, tao2021clustering}, self-supervised learning \cite{zbontar2021barlow, hua2021feature}, class incremental learning \cite{shi2022mimicking} and  federated classification\cite{shi2023understanding}. In light of these, a logical approach to diminish the detrimental effects of data heterogeneity involves addressing the dimensional collapse of the learned representations. To this end, we improve CCFC through the incorporation of a tailored decorrelation regularizer, and call this improved version of CCFC as \textbf{CCFC++}.


Following \cite{shi2023understanding}, we also incorporate an extra regularizer $\frac1{(d')^2}\|\Sigma\|_F^2$ into the loss function $\ell$ (defined in \cref{my_sim}) to avert the collapse of the tail singular values of the covariance matrix $\Sigma$ to zero, mitigating dimensional collapse. Finally, the revised loss function $\ell_r$ is defined as:
\begin{equation}\label{loss}
\ell_{new} = \ell + \frac\eta{(d')^2}\|\Sigma\|_F^2,
\end{equation}
where $\eta$ is the tradeoff hyperparameter.

\begin{table*}[!t]
\centering
\caption{\textbf{NMI of clustering methods in different scenarios.} For each comparison, the best result is highlighted in boldface.}
\renewcommand{\arraystretch}{1.5} 
\tabcolsep 2.8mm 
\begin{tabular}{ccccccccccc}
\hline\hline

\multirow{2}{*}{Dataset} &\multirow{2}{*}{$p$} &\multicolumn{2}{c}{Centralized setting} &\multicolumn{6}{c}{Federated setting}\\ \cmidrule(r){3-4} \cmidrule(r){5-11}
\quad &\quad &\textcolor{mygray}{KM} &\textcolor{mygray}{FCM} &k-FED &FFCM &SDA-FC-KM &SDA-FC-FCM &PPFC-GAN &CCFC &CCFC++\\

\hline
\multirow{5}{*}{MNIST} &0.0 &\multirow{5}{*}{\textcolor{mygray}{0.5304}} &\multirow{5}{*}{\textcolor{mygray}{0.5187}} &0.5081 &0.5157 &0.5133 &0.5141 &0.6582 &0.9236 &\textbf{0.9483}\\
\quad &0.25 &\quad &\quad &0.4879 &0.5264 &0.5033 &0.5063 &0.6392 &0.8152 &\textbf{0.9442} \\
\quad &0.5 &\quad &\quad &0.4515 &0.4693 &0.5118 &0.5055 &0.6721 &0.6718 &\textbf{0.9345} \\
\quad &0.75 &\quad &\quad &0.4552 &0.4855 &0.5196 &0.5143 &\textbf{0.7433} &0.3611 &0.6987 \\
\quad &1.0 &\quad &\quad &0.4142 &0.5372 &0.5273 &0.5140 &\textbf{0.8353} &0.0766 &0.1235 \\

\hline
\multirow{5}{*}{Fashion-MNIST} &0.0 &\multirow{5}{*}{\textcolor{mygray}{0.6070}} &\multirow{5}{*}{\textcolor{mygray}{0.6026}} &0.5932 &0.5786 &0.5947 &0.6027 &0.6091 & 0.6237 & \textbf{0.6321} \\
\quad &0.25 &\quad &\quad &0.5730 &0.5995 &0.6052 &0.5664 &0.5975 & 0.5709 & \textbf{0.6420}\\
\quad &0.5 &\quad &\quad &0.6143 &0.6173 &0.6063 &0.6022 &0.5784 & 0.6023 & \textbf{0.6236}\\
\quad &0.75 &\quad &\quad &0.5237 &\textbf{0.6139} &0.6077 &0.5791 &0.6103 &0.4856 & 0.4966\\
\quad &1.0 &\quad &\quad &0.5452 &0.5855 &0.6065 &0.6026 &\textbf{0.6467} &0.1211 & 0.3187\\

\hline
\multirow{5}{*}{CIFAR-10} &0.0 &\multirow{5}{*}{\textcolor{mygray}{0.0871}} &\multirow{5}{*}{\textcolor{mygray}{0.0823}} &0.0820 &0.0812 &0.0823 &0.0819 &0.1165 &0.2449 &\textbf{0.3447}\\
\quad &0.25 &\quad &\quad &0.0866 &0.0832 &0.0835 &0.0818 &0.1185 &0.2094 &\textbf{0.3363}\\
\quad &0.5 &\quad &\quad &0.0885 &0.0870 &0.0838 &0.0810 &0.1237 &0.2085 &\textbf{0.2461}\\
\quad &0.75 &\quad &\quad &0.0818 &0.0842 &0.0864 &0.0808 &0.1157 &0.1189 &\textbf{0.2033}\\
\quad &1.0 &\quad &\quad &0.0881 &0.0832 &0.0856 &0.0858 &\textbf{0.1318} &0.0639 &0.1125\\

\hline
\multirow{5}{*}{STL-10} &0.0 &\multirow{5}{*}{\textcolor{mygray}{0.1532}} &\multirow{5}{*}{\textcolor{mygray}{0.1469}} &0.1468 &0.1436 &0.1470 &0.1406 &0.1318 &0.2952 &\textbf{0.3169}\\
\quad &0.25 &\quad &\quad &0.1472 &0.1493 &0.1511 &0.1435 &0.1501 &0.1727 &\textbf{0.2743}\\
\quad &0.5  &\quad &\quad &0.1495 &0.1334 &0.1498 &0.1424 &0.1432 &0.2125 &\textbf{0.2702}\\
\quad &0.75 &\quad &\quad &0.1455 &0.1304 &0.1441 &0.1425 &0.1590 &0.1610 &\textbf{0.2480}\\
\quad &1.0  &\quad &\quad &0.1403 &0.1565 &0.1477 &0.1447 &\textbf{0.1629} &0.0711 & 0.0066\\

\hline
count &- &- &- &0 &1 &0 &0  &5 &0 & 14\\
\hline\hline
\end{tabular}\label{NMI}
\end{table*}

\begin{table*}[!t]
\centering
\caption{\textbf{Kappa of clustering methods in different scenarios.} For each comparison, the best result is highlighted in boldface.}
\renewcommand{\arraystretch}{1.5} 
\tabcolsep 2.8mm 
\begin{tabular}{cccccccccccc}
\hline\hline

\multirow{2}{*}{Dataset} &\multirow{2}{*}{$p$} &\multicolumn{2}{c}{Centralized setting} &\multicolumn{6}{c}{Federated setting}\\ \cmidrule(r){3-4} \cmidrule(r){5-11}
\quad &\quad &\textcolor{mygray}{KM} &\textcolor{mygray}{FCM} &k-FED &FFCM &SDA-FC-KM &SDA-FC-FCM &PPFC-GAN &CCFC &CCFC++\\

\hline
\multirow{5}{*}{MNIST} &0.0 &\multirow{5}{*}{\textcolor{mygray}{0.4786}} &\multirow{5}{*}{\textcolor{mygray}{0.5024}} &0.5026 &0.5060 &0.4977 &0.5109 &0.6134 &0.9619 &\textbf{0.9723}\\
\quad &0.25 &\quad &\quad &0.4000 &0.5105 &0.4781 &0.5027 &0.5773 &0.8307 &\textbf{0.9713} \\
\quad &0.5 &\quad &\quad &0.3636 &0.3972 &0.4884 &0.4967 &0.6007 &0.6534 &\textbf{0.9653} \\
\quad &0.75 &\quad &\quad &0.3558 &0.4543 &0.4926 &0.5021 &\textbf{0.6892} &0.3307 &0.6198 \\
\quad &1.0 &\quad &\quad &0.3386 &0.5103 &0.5000 &0.5060 &\textbf{0.7884} &0.0911 &0.1445 \\

\hline
\multirow{5}{*}{Fashion-MNIST} &0.0 &\multirow{5}{*}{\textcolor{mygray}{0.4778}} &\multirow{5}{*}{\textcolor{mygray}{0.5212}} &0.4657 &0.4974 &0.4918 &0.4918 &0.4857 & 0.6411 & \textbf{0.6460}\\
\quad &0.25 &\quad &\quad &0.5222 &0.5180 &0.4918 &0.4918 &0.4721 &0.5261 & \textbf{0.6547}\\
\quad &0.5 &\quad &\quad &0.4951 &0.4974 &0.4918 &0.4918 &0.4552 &0.5929 & \textbf{0.6277}\\
\quad &0.75 &\quad &\quad &0.4240 &\textbf{0.4995} &0.4918 &0.4918 &0.4774 &0.3945 & 0.4754\\
\quad &1.0 &\quad &\quad &0.3923 &0.4672 &0.4918 &0.4918 &\textbf{0.5745} &0.1434 & 0.2777\\

\hline
\multirow{5}{*}{CIFAR-10} &0.0 &\multirow{5}{*}{\textcolor{mygray}{0.1347}} &\multirow{5}{*}{\textcolor{mygray}{0.1437}} &0.1305 &0.1439 &0.1275 &0.1283 &0.1426 &0.2854 &\textbf{0.3760}\\
\quad &0.25 &\quad &\quad &0.1366 &0.1491 &0.1275 &0.1376 &0.1400 &0.2281 &\textbf{0.3758}\\
\quad &0.5 &\quad &\quad &0.1252 &0.1316 &0.1307 &0.1411 &0.1443 &0.2214 &\textbf{0.3048}\\
\quad &0.75 &\quad &\quad &0.1303 &0.1197 &0.1360 &0.1464 &0.1358 &0.1214 &\textbf{0.2542}\\
\quad &1.0 &\quad &\quad &0.1147 &0.1237 &0.1341 &0.1494 &\textbf{0.1499} &0.1047 &0.1415\\

\hline
\multirow{5}{*}{STL-10} &0.0 &\multirow{5}{*}{\textcolor{mygray}{0.1550}} &\multirow{5}{*}{\textcolor{mygray}{0.1602}} &0.1390 &0.1514 &0.1533 &0.1505 &0.1557 &0.1687 &\textbf{0.2571}\\
\quad &0.25 &\quad &\quad &0.1361 &0.1479 &0.1448 &0.1527 &0.1611 &0.1422 &\textbf{0.2551}\\
\quad &0.5 &\quad &\quad &0.1505 &0.1112 &0.1377 &0.1620 &0.1415  &0.1407 &\textbf{0.2427}\\
\quad &0.75 &\quad &\quad &0.1256 &0.1001 &0.1513 &0.1603 &0.1813 &0.1133 &\textbf{0.2332}\\
\quad &1.0 &\quad &\quad &0.1328 &0.1351 &0.1527 &0.1553 &\textbf{0.1868}  &0.0519 &0.0081\\

\hline
count &- &- &- &0 &1 &0  &0  &5 &0 &14\\
\hline\hline
\end{tabular}\label{Kappa}
\end{table*}

\section{Experiments}
\subsection{Experimental setup}
We evaluate CCFC++ on federated datasets simulated on MNIST (70,000 images with 10 true clusters), Fashion-MNIST (70,000 images with 10 true clusters), CIFAR-10 (60,000 images with 10 true clusters), and STL-10 (13,000 images with 10 true clusters). The federated data simulation method has already been introduced in \cref{cg}. The evaluation metrics are normalized mutual information (NMI) \cite{strehl2002cluster} and Kappa \cite{liu2019evaluation}.

To focus on the efficacy of the decorrelation regularizer and to avoid excessive hyperparameter tuning, we adhere to the default configurations from CCFC \cite{yan2024ccfc} for aspects such as the network architecture of the cluster-contrastive model, the trade-off hyperparameter $\lambda$, the latent representation dimension, the learning rate and the optimizer, while solely tuning the tradeoff hyperparameter $\eta$. The tradeoff hyperparameter $\eta$ is set to 0.01 for MNIST, and 0.1 for Fashion-MNIST, CIFAR-10 and STL-10. The code will be made available.

\subsection{Efficacy of the decorrelation regularizer}
Since the goal of combining federated clustering with contrastive learning is to improve clustering performance by learning more clustering-friendly representations, we validate the efficacy of the decorrelation regularizer from two aspects: mitigating dimensional collapse and learning clustering-friendly representations.

To this end, we first perform CCFC and CCFC++ under different simulated federated scenarios on MNIST, respectively. Then, we plot the singular values of the covariance matrix of the learned representations, and  visualize the learned representations using t-SNE \cite{van2008visualizing}.
As shown in \cref{case4}, CCFC++ with the decorrelation regularizer effectively mitigates the dimensional collapse in CCFC under different federated scenarios, leading to more clustering-friendly representations.

\subsection{Effectiveness of CCFC++}
For comprehensive comparison, we additionally select five cutting-edge baselines, including k-FED \cite{dennis2021heterogeneity}, FFCM \cite{stallmann2022towards}, SDA-FC-KM \cite{sdafc}, SDA-FC-FCM \cite{sdafc} and PPFC-GAN \cite{ppfcgan}. To avoid over-tuning, the hyperparameter settings for each method are identical across different scenarios within the same dataset.

As shown in Tables \cref{NMI} and \cref{Kappa}, one can see that: 1) Both NMI and Kappa corroborate the dominant superiority of CCFC++ in most cases. 2) Benefiting from the decorrelation regularizer, CCFC++ effectively mitigates the detrimental effects of data heterogeneity. The most striking case occurs on MNIST under the data heterogeneity condition of $p = 0.75$, witnessing improvements of up to 0.32 for NMI and 0.27 for Kappa. 3) Interestingly, the decorrelation regularizer, while designed for heterogeneous data, also significantly enhances performance in homogeneous scenarios ($p = 0$), suggesting potential applications in centralized clustering methods.

\subsection{Hyperparameter sensitivity analysis}
Given that this work centers on counteracting the adverse effects of heterogeneous data on clustering performance through feature decorrelation, we mainly analyze the sensitivity of CCFC++ to the tradeoff hyperparameter $\eta$, with the tradeoff hyperparameter $\lambda$ adhering to the configuration specified in \cite{yan2024ccfc}.

\cref{case5} illustrates that CCFC++ remains robust across a wide range of $\eta$ for each federated scenario. Notably, the optimal $\eta$ varies under distinct federated scenarios, suggesting the reported improvements in Tables \cref{NMI} and \cref{Kappa} are understated.

\begin{figure}[!t]
\centering
\subfigure[NMI]{
\includegraphics[height = 4cm, width = 3.9cm]{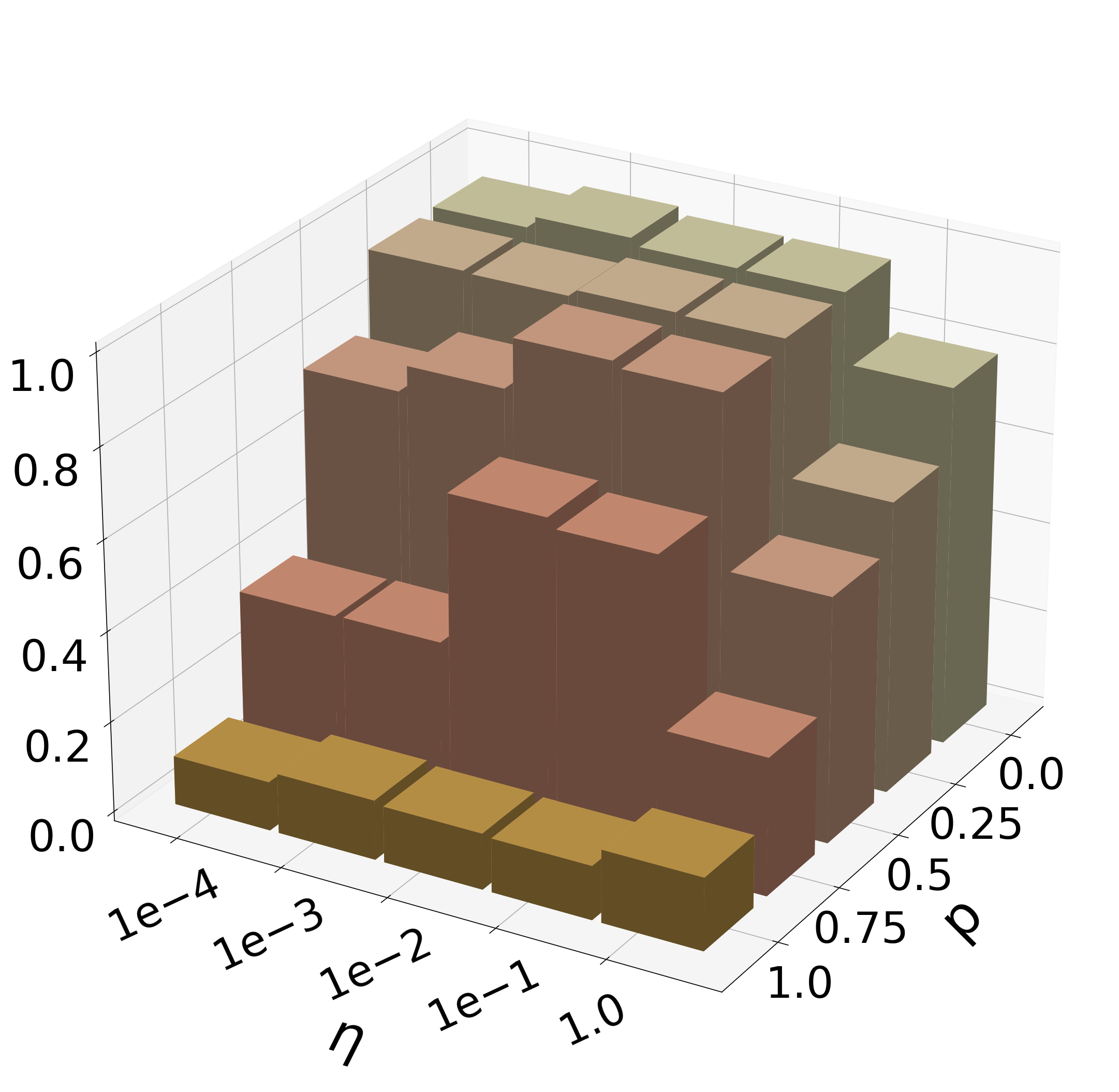}}
\,
\subfigure[Kappa]{
\includegraphics[height = 4cm, width = 3.9cm]{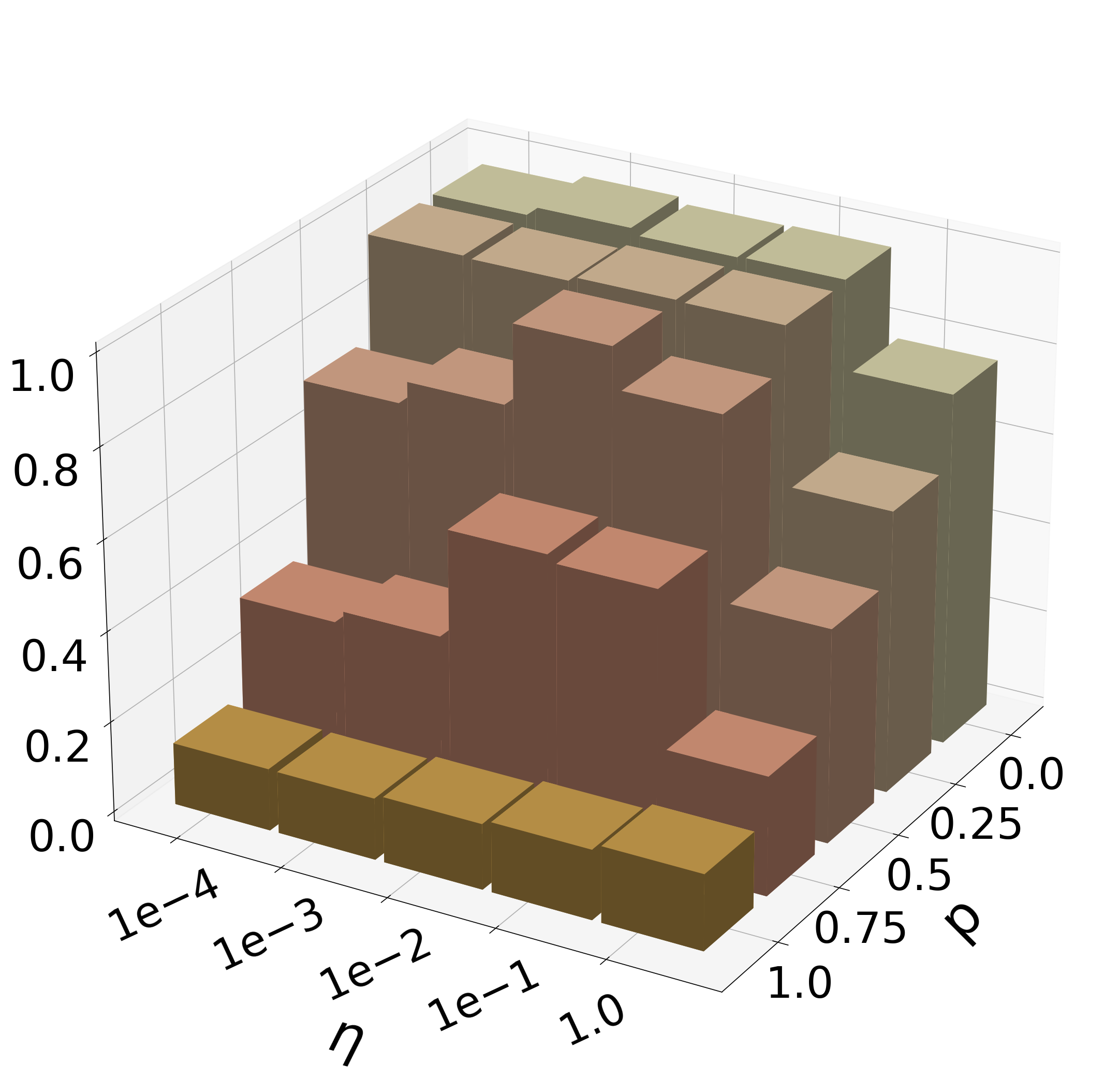}}

\caption{\textbf{Sensitivity of CCFC++ to the hyperparameter $\eta$ under different simulated federated scenarios on MNIST}.}
\label{case5}
\end{figure}

\begin{figure*}[!t]
\centering
\includegraphics[height = 8.5cm, width = 18cm]{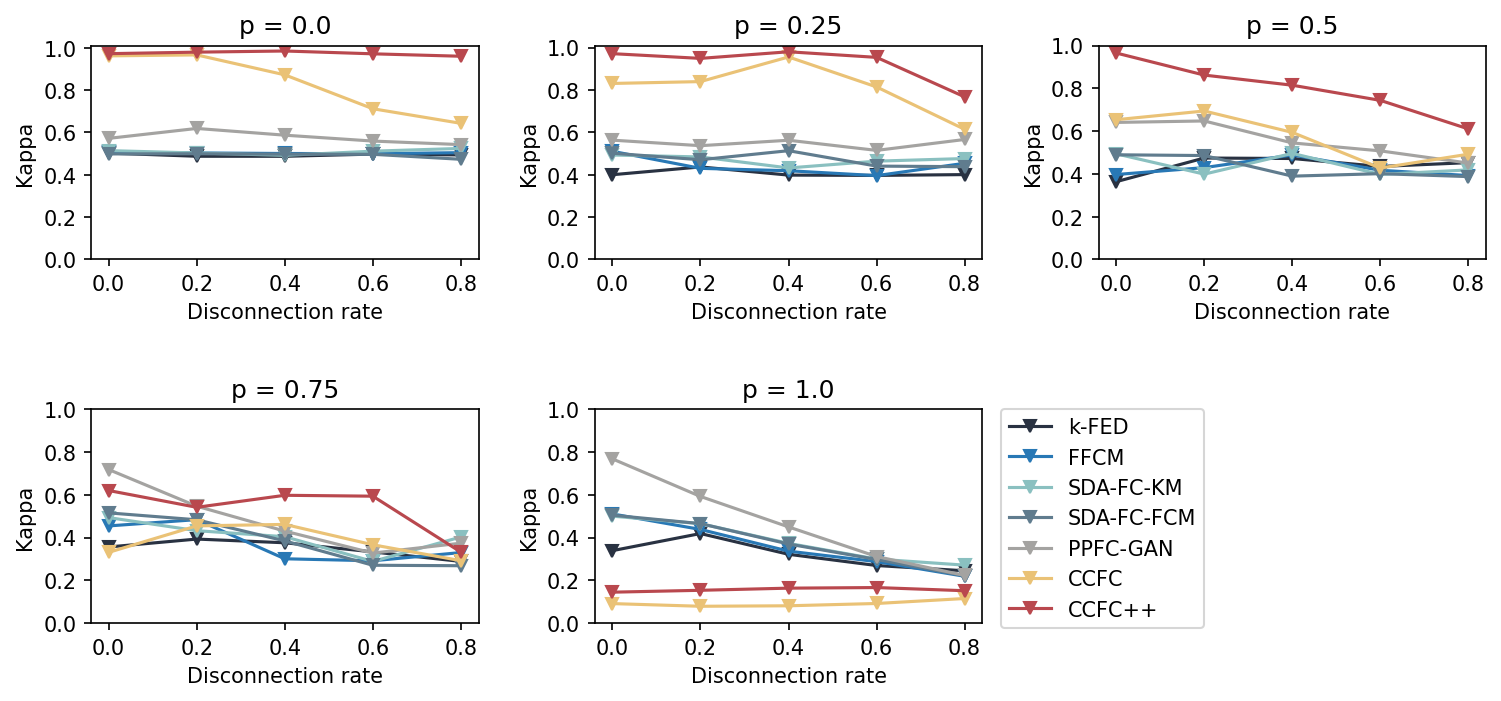}

\caption{\textbf{Sensitivity of CCFC++ to the device failures under different simulated federated scenarios on MNIST}.}
\label{case6}
\end{figure*}

\subsection{Device failures}
In practice, beyond data heterogeneity, systems heterogeneity is also a core concern in federated learning \cite{li2020federated, yan2024ccfc}. This scenario is characterized by disparate computational, storage and communication capacities among clients, and some clients are unable to engage in model training or may experience disconnection from the server mid-training, leading to poor and unrobust model performance. Hence, from a pragmatic perspective, it becomes crucial to explore how CCFC++ responds to device failures.

Following \cite{yan2024ccfc}, we use the \textbf{disconnection rate} to measure the ratio of disconnected clients relative to all clients, with only the connected clients partaking in the training throughout the entire process. As shown in \cref{case6}, one can see that: 1) CCFC++ exhibits a dominant superiority in most cases. 2) Benefiting from the decorrelation regularizer, CCFC++ improves both the clustering performance and robustness of CCFC in handling device failures. 3) Occasionally, device failures can even improve the clustering performance of some methods, suggesting that strategic subsampling could potentially further boost the clustering performance.

In summary: 1) The decorrelation regularizer effectively mitigates the dimensional collapse in CCFC under different federated scenarios, leading to more clustering-friendly representations. 2) This regularizer demonstrates beneficial in managing both data and systems heterogeneity. 3) CCFC++ exhibits a dominant superiority in most cases of the simulated scenarios. 4) CCFC++ demonstrates robustness to varying values of the tradeoff hyperparameter $\eta$ for each fixed federated scenario.

\section{Conclusion}
In this work, we first analyse how heterogeneous data affects CCFC. Both empirical and theoretical investigations reveal that increased data heterogeneity exacerbates dimensional collapse in CCFC. In light of these, we improve CCFC through the incorporation of a decorrelation regularizer.
Comprehensive experiments demonstrate the effectiveness of the regularizer and the improved CCFC.

We hope our work will serve as a catalyst for future research in FC or other unsupervised federated learning domains, inspiring fellow researchers and practitioners to tackle similar challenges.

\bibliographystyle{IEEEtran}
\bibliography{references}

\newpage

\onecolumn

{\appendix[Proof of the Theorem 3.1]


\begin{table}[!t]
\centering
\caption{Notations}
\renewcommand{\arraystretch}{1.8} 
\tabcolsep 11mm 
\begin{tabular}{cc}
\hline\hline
Notation   &Explanation    \\\hline
$k$        &The number of clusters\\
$n_c$      &The number of training samples in cluster $c$, $c \in [k] = \{1,\, 2,\, \cdots,\, k\}$\\
$d$        &The dimension of training samples\\

$X^{(c)}$  &The collection of samples in cluster $c$, $X^{(c)} = [x_{1}^{(c)},\, x_{2}^{(c)},\, \cdots,\, x_{n_c}^{(c)}] \in \mathbb{R}^{d \times n_c}$\\
$x_{oi}^{(c)}$  &The $o$-th feature of $x_{i}^{(c)}$, $x_{oi}^{(c)} \in \mathbb{R}$, $o \in [d]$, $i \in [n_c]$\\
$x_{o\cdot}^{(c)}$  &The collection of the $o$-th features in cluster $c$ (i.e. the $o$-th row vector of $X^{(c)}$), $x_{o\cdot}^{(c)}\in \mathbb{R}^{1 \times n_c}$\\

$\Pi(t)$     &The weight matrix of the local encoder at the $t$-th optimization step\\
$\Pi^{(g)}$  &The weight matrix of the global encoder, which remains fixed during the local training process\\
$\Phi(t)$    &The weight matrix of the local predictor at the $t$-th optimization step\\
$\Phi^{(g)}$ &The weight matrix of the global predictor, which remains fixed during the local training process\\

$d'$        &The dimension of the latent representations\\
$Z^{(c)}(t)$ &The local latent representation of $X^{(c)}$, $Z^{(c)}(t) = \Pi(t) X^{(c)} = [z_{1}^{(c)}(t),\, z_{2}^{(c)}(t),\, \cdots,\, z_{n_c}^{(c)}(t)] \in \mathbb{R}^{d' \times n_c}$\\
$P^{(c)}(t)$ &The local predictions of $X^{(c)}$, $P^{(c)}(t) = \Phi(t) Z^{(c)}(t) = [p_{1}^{(c)}(t),\, p_{2}^{(c)}(t),\, \cdots,\, p_{n_c}^{(c)}(t)] \in \mathbb{R}^{d' \times n_c}$\\
$P^{(g,\, c)}$ &The global predictions of $X^{(c)}$, $P^{(g,\, c)} = \Phi^{(g)}\Pi^{(g)} X^{(c)} = [p_{1}^{(g,\, c)},\, p_{2}^{(g,\, c)},\, \cdots,\, p_{n_c}^{(g,\, c)}] \in \mathbb{R}^{d' \times n_c}$\\

$\sigma_{\tau}^{\Pi}(t)$  &The $\tau$-th largest singular value of $\Pi(t)$\\
$u_{\tau}^{\Pi}(t)$       &The $\tau$-th left singular vector of $\Pi(t)$\\
$v_{\tau}^{\Pi}(t)$       &The $\tau$-th right singular vector of $\Pi(t)$\\

$\sigma_{\tau}^{\Phi}(t)$  &The $\tau$-th largest singular value of $\Phi(t)$\\
$u_{\tau}^{\Phi}(t)$       &The $\tau$-th left singular vector of $\Phi(t)$\\
$v_{\tau}^{\Phi}(t)$       &The $\tau$-th right singular vector of $\Phi(t)$\\
\hline\hline
\end{tabular}
\label{notations}
\end{table}

\label{prove}
Before proving the \cref{theory}, we first summarize some notations used in both the main text and this appendix, and introduce two lemmas from Shi et al. \cite{shi2023understanding}. Refer to \cref{notations} for the notions, with the lemmas delineated below:

\begin{lemma}
Given $L$ successive linear layers in a neural network, each corresponds to a weight matrix $W_i$ $(i \in [L])$. At the optimization time step $t$, the product of these matrices is denoted as $\Pi(t) = W_{L}(t)W_{L-1}(t)\cdots W_{1}(t)$. Assuming at the time step $0$, $W_{i}(0)(W_{i}(0))^{\top}=(W_{i+1}(0))^{\top}W_{i+1}(0)$ holds for any layer $i \in [L-1]$. Then, the gradient descent dynamics of $\Pi(t)$ satisfies:
\begin{equation}
\dot{\Pi}(t)=-\sum_{i=1}^{L}\left[\Pi(t)\Pi(t)^\top\right]^{\frac{L-i}L}
\frac{\partial\ell(\Pi(t))}{\partial\Pi}
\left[\Pi(t)^\top\Pi(t)\right]^{\frac{i-1}L},
\end{equation}
where $\left[\cdot \right]^{\frac{L-i}L}$ and $\left[\cdot \right]^{\frac{i-1}L}$ are fractional power operators.
\label{lem1}
\end{lemma}

\begin{lemma}
Under gradient descent dynamics with infinitesimally small learning rate, the $\tau$-th largest singular value $\sigma_{\tau}$ of the weight matrix $W$ evolves as:
\begin{equation}
\dot{\sigma}_{\tau}(t)=(u_{\tau}(t))^\top\dot{W}(t)v_{\tau}(t),
\end{equation}
where $u_{\tau}(t)$ and $v_{\tau}(t)$ are the $\tau$-th left and right singular vectors of the weight matrix $W$.
\label{lem2}
\end{lemma}

Based on these lemmas, we can derive the theorem below.

\begin{theorem}
Under assumptions \cref{ass1} and \cref{ass2}, the gradient descent dynamics of the $\tau$-th largest singular value $\sigma_{\tau}^{\Pi}(t)$ of $\Pi(t)$ can be expressed as:
\begin{align}
\dot{\sigma}_{\tau}^{\Pi}(t)  = &L_1(\sigma_\tau^{\Pi}(t))^{2-\frac{2}{L_1}}
\sqrt{(\sigma_{\tau}^{\Pi}(t))^{\frac2{L_1}} + C} \nonumber\\
& \times
(u_{\tau}^{\Phi}(t))^\top
\bar{Q}(t)
v_\tau^{\Pi}(t),
\end{align}
where $u_{\tau}^{\Phi}(t)$ is the $\tau$-th left singular vector of $\Phi(t)$, $v_{{\tau}}^{\Pi}(t)$ is the $\tau$-th right singular vector of $\Pi(t)$, $C$ is a constant,
\begin{equation}
\bar{Q}(t) = \frac1k\sum_{c = 1}^{k}Q^{(c)}(t) (X^{(c)})^\top,
\end{equation}
the element $q_{ri}^{(c)}(t) \in \mathbb{R}$ located in the $r$-th $(r \in \{1,\, 2,\, \cdots,\, d'\})$ row and $i$-th $(i \in [n_c])$ column of $Q^{(c)}(t) \in \mathbb{R}^{d' \times n_c}$ is expressed as:
\begin{equation}
q_{ri}^{(c)}(t) = \left(\frac{1}{n_c^2}\sum_{j = 1}^{n_c} \hat{z}_{rj}^{(c)}(t) + \frac{\lambda}{n_c} \hat{p}_{ri}^{(g,\, c)}   \right)
\cdot \frac{1 - (\hat{p}_{ri}^{(c)}(t))^2}{\|p_{i}^{(c)}(t)\|_2},
\end{equation}
$\hat{p}_{ri}^{(c)}(t) = \frac{p_{ri}^{(c)}(t)}{\|p_{i}^{(c)}(t)\|_2} \in \mathbb{R}$,
$\hat{z}_{rj}^{(c)}(t) = \frac{z_{rj}^{(c)}(t)}{\|z_{j}^{(c)}(t)\|_2} \in \mathbb{R}$, and $\hat{p}_{ri}^{(g,\, c)} = \frac{p_{ri}^{(g,\, c)}}{\|p_{i}^{(g,\, c)}\|_2} \in \mathbb{R}$.
\end{theorem}

\begin{proof} (\cref{theory})
For simplicity, during this proof, we omit the notation for the time optimization step $t$, e.g. $\Pi(t)$ is represented simply as $\Pi$.


Given a cluster-contrastive model with $L_1 + L_2$ linear layers, and $k$ clusters obtained by labeling the local data with the index corresponding to the nearest global cluster centroid. Based on the notations denoted in \cref{notations}, the loss function defined in \cref{my_sim} can be rewritten as:
\begin{equation}
\ell(\Pi,\, \Phi) = -\sum_{c = 1}^{k}\left[\frac{1}{kn_c^2}\sum_{i = 1}^{n_c}\sum_{j = 1}^{n_c}\sum_{r = 1}^{d'} \hat{p}_{ri}^{(c)}\cdot stopgrad(\hat{z}_{rj}^{(c)}) + \frac{\lambda}{kn_c}\sum_{i = 1}^{n_c}\sum_{r = 1}^{d'} \hat{p}_{ri}^{(c)}\cdot stopgrad(\hat{p}_{ri}^{(g,\, c)})\right],
\end{equation}
where $\hat{p}_{ri}^{(c)} = \frac{p_{ri}^{(c)}}{\|p_{i}^{(c)}\|_2} \in \mathbb{R}$,
\begin{equation}
p_{ri}^{(c)} = w^\Phi_{r\cdot}\Pi(t) x_{i}^{(c)} = \sum_{s = 1}^{d'}\sum_{o = 1}^{d} w^\Phi_{rs}w^\Pi_{so}x_{oi}^{(c)} \in \mathbb{R}
\end{equation}
represents the element located in the $r$-th row and $i$-th column of the local prediction matrix $P^{(c)}$, $w^\Phi_{r\cdot}$ is the $r$-th row vector of the weight matrix $\Phi$, $w^\Phi_{rs} \in \mathbb{R}$ is the element located in the $r$-th row and $i$-th column of $\Phi$, $w^\Pi_{so} \in \mathbb{R}$ is the one of $\Pi$ and $x_{oi}^{(c)} \in \mathbb{R}$ is the one of $X^{(c)}$. Similarly, $\hat{z}_{rj}^{(c)} = \frac{z_{rj}^{(c)}}{\|z_{j}^{(c)}\|_2} \in \mathbb{R}$ and $\hat{p}_{ri}^{(g,\, c)} = \frac{p_{ri}^{(g,\, c)}}{\|p_{i}^{(g,\, c)}\|_2} \in \mathbb{R}$, $z_{rj}^{(c)}$ represents the element located in the $r$-th row and $j$-th column of the local representation matrix $Z^{(c)}$, and $p_{ri}^{(g,\, c)}$ represents the one of the global prediction matrix $P^{(g,\, c)}$.



Then, the gradient descent dynamics of  $\Pi$ and $\Phi$ are respectively denoted as:
\begin{align}
\dot{\Pi}(t) &= -\frac{\partial \ell(\Pi,\, \Phi)}{\partial \Pi}, \label{d1}\\
\dot{\Phi}(t) &= -\frac{\partial \ell(\Pi,\, \Phi)}{\partial \Phi} \label{d2}.
\end{align}
More specifically, by the chain rule, the gradient of $\ell(\Pi,\, \Phi)$ with respect to $w^\Pi_{so}$ can be derived as:
\begin{align}
\frac{\partial \ell(\Pi,\, \Phi)}{\partial w^\Pi_{so}}
&= \sum_{c = 1}^{k} \sum_{i = 1}^{n_c} \sum_{r = 1}^{d'}
\frac{\partial \ell(\Pi,\, \Phi)}{\partial \hat{p}_{ri}^{(c)}}
\cdot \frac{\partial \hat{p}_{ri}^{(c)}}{\partial p_{ri}^{(c)}}
\cdot \frac{\partial p_{ri}^{(c)}}{\partial w^\Pi_{so}}
\tag{The stopgrad operation treats $\hat{z}_{rj}^{(c)}$ and $\hat{p}_{rj}^{(g,\, c)}$ as constants}
\\
&= -\sum_{c = 1}^{k}\sum_{i = 1}^{n_c}\sum_{r = 1}^{d'} \left(\frac{1}{kn_c^2}\sum_{j = 1}^{n_c} \hat{z}_{rj}^{(c)} + \frac{\lambda}{kn_c} \hat{p}_{ri}^{(g,\, c)}   \right)
\cdot \frac{1 - (\hat{p}_{ri}^{(c)})^2}{\|p_{i}^{(c)}\|_2}
\cdot w^\Phi_{rs}x_{oi}^{(c)}
\nonumber\\
&= -\frac1k\sum_{c = 1}^{k}\sum_{i = 1}^{n_c}\sum_{r = 1}^{d'} q_{ri}^{(c)} w^\Phi_{rs} x_{oi}^{(c)}
\tag{Let $q_{ri}^{(c)} = \left(\frac{1}{n_c^2}\sum_{j = 1}^{n_c} \hat{z}_{rj}^{(c)} + \frac{\lambda}{n_c} \hat{p}_{ri}^{(g,\, c)}   \right)
\cdot \frac{1 - (\hat{p}_{ri}^{(c)})^2}{\|p_{i}^{(c)}\|_2} \in \mathbb{R}$}
\\
&= -\frac1k\sum_{c = 1}^{k}\sum_{r = 1}^{d'}w^\Phi_{rs} \sum_{i = 1}^{n_c} q_{ri}^{(c)} x_{oi}^{(c)}
\nonumber\\
&= -\frac1k\sum_{c = 1}^{k}\sum_{r = 1}^{d'}w^\Phi_{rs} q_{r\cdot}^{(c)} (x_{o\cdot}^{(c)})^\top
\tag{$q_{r\cdot}^{(c)} =
[q_{r1}^{(c)},\, q_{r2}^{(c)},\,  \cdots, q_{rn_c}^{(c)}] \in \mathbb{R}^{1\times n_c}$}
\\
&= -\frac1k\sum_{c = 1}^{k}(w^\Phi_{s})^\top Q^{(c)} (x_{o\cdot}^{(c)})^\top,
\end{align} 
where $w^\Phi_{s} \in \mathbb{R}^{n_c\times 1}$ is the $s$-th column vector of the weight matrix $\Phi$. Then, we can have
\begin{align}
\frac{\partial \ell(\Pi,\, \Phi)}{\partial \Pi} = -\frac1k\sum_{c = 1}^{k} \Phi^\top Q^{(c)} (X^{(c)})^\top = -\Phi^\top \bar{Q},
\label{d_result1}
\end{align}
where $\bar{Q} = \frac1k\sum_{c = 1}^{k}Q^{(c)} (X^{(c)})^\top$, $Q^{(c)} = [(q_{1\cdot}^{(c)})^\top,\, (q_{2\cdot}^{(c)})^\top,\,  \cdots, (q_{d'\cdot}^{(c)})^\top]^\top \in \mathbb{R}^{d' \times n_c}$. Based on the \cref{lem1} and \cref{d_result1}, the gradient descent dynamics of $\Pi$ (\cref{d1}) can be derived as:
\begin{align}
\dot{\Pi} &= -\sum_{i=1}^{L_1}\left[\Pi\Pi^\top\right]^{\frac{L_1-i}L_1}
\frac{\partial\ell(\Pi)}{\partial\Pi}
\left[\Pi^\top\Pi\right]^{\frac{i-1}{L_1}}
\nonumber\\
&= \sum_{i=1}^{L_1}\left[\Pi\Pi^\top\right]^{\frac{L_1-i}L_1}
\Phi^\top \bar{Q}
\left[\Pi^\top\Pi\right]^{\frac{i-1}{L_1}}.
\end{align}
Next, under gradient descent dynamics with infinitesimally small learning rate, the $\tau$-th largest singular value $\sigma_{\tau}^{\Pi}$ of the weight matrix $\Pi$ evolves as:
\begin{align}
\dot{\sigma}_{\tau}^{\Pi} &= (u_{\tau}^{\Pi})^\top\dot{\Pi}v_{\tau}^{\Pi}
\tag{\cref{lem2}}
\\
&= (u_{\tau}^{\Pi})^\top
\sum_{i=1}^{L_1}\left[\Pi\Pi^\top\right]^{\frac{L_1-i}{L_1}}
\Phi^\top \bar{Q}
\left[\Pi^\top\Pi\right]^{\frac{i-1}{L_1}}
v_{\tau}^{\Pi}
\nonumber\\
&= L_1(\sigma_\tau^{\Pi})^{2-\frac{2}{L_1}}
(u_\tau^{\Pi})^\top
\Phi^\top \bar{Q}
v_\tau^{\Pi}
\tag{SVD on $\Pi$: $\Pi = \sum_{\tau} \sigma_{\tau}^{\Pi} u_{\tau}^{\Pi} (v_{\tau}^{\Pi})^\top$} \\
&= L_1(\sigma_\tau^{\Pi})^{2-\frac{2}{L_1}}
\sum_{\tau'}\sigma_{\tau'}^{\Phi} (u_\tau^{\Pi})^\top v_{\tau'}^{\Phi}(u_{\tau'}^{\Phi})^\top
\bar{Q}
v_\tau^{\Pi}
\tag{SVD on $\Phi$: $\Phi = \sum_{\tau'} \sigma_{\tau'}^{\Phi} u_{\tau'}^{\Phi} (v_{\tau'}^{\Phi})^\top$}
\\
&= L_1(\sigma_\tau^{\Pi})^{2-\frac{2}{L_1}}
\sigma_{\tau}^{\Phi} (u_{\tau}^{\Phi})^\top
\bar{Q}
v_\tau^{\Pi}.
&&\text{(\cref{ass2})}
\label{eq1}
\end{align}

Similarly, by the chain rule, the gradient of $\ell(\Pi,\, \Phi)$ with respect to $w^\Phi_{rs}$ can be derived as:
\begin{align}
\frac{\partial \ell(\Pi,\, \Phi)}{\partial w^\Phi_{rs}}
&= \sum_{c = 1}^{k} \sum_{i = 1}^{n_c}
\frac{\partial \ell(\Pi,\, \Phi)}{\partial \hat{p}_{ri}^{(c)}}
\cdot \frac{\partial \hat{p}_{ri}^{(c)}}{\partial p_{ri}^{(c)}}
\cdot \frac{\partial p_{ri}^{(c)}}{\partial w^\Phi_{rs}}
\tag{The stopgrad operation treats $\hat{z}_{rj}^{(c)}$ and $\hat{p}_{rj}^{(g,\, c)}$ as constants}
\\
&= -\sum_{c = 1}^{k}\sum_{i = 1}^{n_c} \left(\frac{1}{kn_c^2}\sum_{j = 1}^{n_c} \hat{z}_{rj}^{(c)} + \frac{\lambda}{kn_c} \hat{p}_{ri}^{(g,\, c)}   \right)
\cdot \frac{1 - (\hat{p}_{ri}^{(c)})^2}{\|p_{i}^{(c)}\|_2}
\cdot \sum_{o = 1}^{d} w^\Pi_{so}x_{oi}^{(c)}
\nonumber\\
&= -\frac1k\sum_{c = 1}^{k}\sum_{i = 1}^{n_c} q_{ri}^{(c)} w^\Pi_{s\cdot} x_{i}^{(c)} \nonumber \\
&= -\frac1k\sum_{c = 1}^{k}w^\Pi_{s\cdot} \sum_{i = 1}^{n_c} x_{i}^{(c)}q_{ri}^{(c)} \nonumber\\
&= -\frac1k\sum_{c = 1}^{k}w^\Pi_{s\cdot} X^{(c)}(q_{r\cdot}^{(c)})^\top
\nonumber     \\
&= -\frac1k\sum_{c = 1}^{k}q_{r\cdot}^{(c)}(X^{(c)})^\top (w^\Pi_{s\cdot})^\top.
\end{align}
Then, we can have
\begin{equation}
\frac{\partial \ell(\Pi,\, \Phi)}{\partial \Phi} = -\frac1k\sum_{c = 1}^{k}
Q^{(c)}(X^{(c)})^\top \Pi^\top = -\bar{Q}\Pi^\top,
\end{equation}
and
\begin{equation}
\dot{\Phi}(t) = -\frac{\partial \ell(\Pi,\, \Phi)}{\partial \Phi} =
\bar{Q} \Pi^\top.
\end{equation}
Next, under gradient descent dynamics with infinitesimally small learning rate, the $\tau'$-th largest singular value $\sigma_{\tau'}^{\Phi}$ of the weight matrix $\Phi$ evolves as:
\begin{align}
\dot{\sigma}_{\tau'}^{\Phi} &= (u_{\tau'}^{\Phi})^\top\dot{\Phi}(t)v_{\tau'}^{\Phi}
\tag{\cref{lem2}}\\
&= (u_{\tau'}^{\Phi})^\top
\bar{Q} \Pi^\top
v_{\tau'}^{\Phi} \nonumber\\
&= \sum_{\tau}\sigma_{\tau}^{\Pi}
(u_{\tau'}^{\Phi})^\top
\bar{Q}
v_{\tau}^{\Pi}(u_{\tau}^{\Pi})^\top
v_{\tau'}^{\Phi}
\tag{SVD on $\Pi$} \\
&= \sigma_{\tau}^{\Pi}
(u_{\tau}^{\Phi})^\top
\bar{Q}
v_{\tau}^{\Pi}
&&\text{(\cref{ass2})}
\label{eq2}
\end{align}

Combining equations (\cref{eq1}) and (\cref{eq2}), we can have
\begin{equation}
2\sigma_{\tau}^{\Phi}\dot{\sigma}_{\tau}^{\Phi} = \frac2{L_1}\dot{\sigma}_{\tau}^{\Pi}(\sigma_{\tau}^{\Pi})^{\frac2{L_1}-1}.
\end{equation}
By integrating both sides, we have
\begin{equation}
(\sigma_{\tau}^{\Phi})^2=(\sigma_{\tau}^{\Pi})^{\frac2{L_1}} + C,
\end{equation}
and
\begin{equation}
\sigma_{\tau}^{\Phi} = \sqrt{(\sigma_{\tau}^{\Pi})^{\frac2{L_1}} + C},
\end{equation}
where $C$ is a constant.

Finally, \cref{eq1} can be further expressed as:
\begin{equation}
\dot{\sigma}_{\tau}^{\Pi}  = L_1(\sigma_\tau^{\Pi})^{2-\frac{2}{L_1}}
\sqrt{(\sigma_{\tau}^{\Pi})^{\frac2{L_1}} + C} (u_{\tau}^{\Phi})^\top
\bar{Q}
v_\tau^{\Pi}.
\end{equation}
\end{proof}
}

\end{document}